\documentclass{article}

\usepackage{microtype}
\usepackage{graphicx}
\usepackage{subfigure}
\usepackage{booktabs} 
\usepackage{color}
\usepackage{xcolor}

\usepackage{hyperref}

\definecolor{tableblue}{HTML}{AFC2FF}
\definecolor{tablelightblue}{HTML}{E1E8FF}
\definecolor{mydarkblue}{rgb}{0,0.08,0.45}
\definecolor{mydarkgreen}{RGB}{0, 139, 69}
\definecolor{mycyan}{cmyk}{.3,0,0,0}
\definecolor{dkgreen}{rgb}{0,0.6,0}
\definecolor{dkred}{rgb}{0.6,0,0}
\definecolor{dkblue}{rgb}{0,0,0.6}
\definecolor{purple}{rgb}{0.5,0,0.5}

\usepackage[accepted]{icml2024}

\usepackage{amsmath}
\usepackage{amssymb}
\usepackage{mathtools}
\usepackage{amsthm}
\usepackage{multicol}
\usepackage{multirow}
\usepackage{colortbl}
\usepackage{enumitem}
\usepackage[capitalize,noabbrev]{cleveref}

\newcommand{\cdummy}{\cdot}
\newcommand{\mathd}{\mathrm{d}}
\newcommand{\tmmathbf}[1]{\ensuremath{\boldsymbol{#1}}}
\newcommand{\tmop}[1]{\ensuremath{\operatorname{#1}}}
\newcommand{\junz}[1]{{\color{blue}{[\textbf{jz}: #1]}}}

\theoremstyle{plain}
\newtheorem{theorem}{Theorem}[section]

\newtheorem{lemma}[theorem]{Lemma}

\theoremstyle{definition}
\newtheorem{definition}[theorem]{Definition}

\theoremstyle{remark}

\usepackage[textsize=tiny]{todonotes}

\icmltitlerunning{DPOT: Auto-Regressive Denoising Operator Transformer for Large-Scale PDE Pre-Training}

\begin{document}

\twocolumn[
\icmltitle{DPOT: Auto-Regressive Denoising Operator Transformer \\for Large-Scale PDE Pre-Training}




\begin{icmlauthorlist}
\icmlauthor{Zhongkai Hao}{cs,ee}
\icmlauthor{Chang Su}{cs}
\icmlauthor{Songming Liu}{cs}
\icmlauthor{Julius Berner}{cl}
\icmlauthor{Chengyang Ying}{cs}
\icmlauthor{Hang Su}{cs}
\icmlauthor{Anima Anandkumar}{cl}
\icmlauthor{Jian Song}{ee}
\icmlauthor{Jun Zhu}{cs,re}
\end{icmlauthorlist}

\icmlaffiliation{cs}{Dept. of Comp. Sci. \& Techn., Institute for AI, BNRist Center, Tsinghua-Bosch Joint ML Center, Tsinghua University}
\icmlaffiliation{ee}{Dept. of EE, Tsinghua University}
\icmlaffiliation{re}{RealAI}
\icmlaffiliation{cl}{Caltech}

\icmlcorrespondingauthor{Jun Zhu}{dcszj@tsinghua.edu.cn}

\icmlkeywords{Machine Learning, ICML}

\vskip 0.3in
]



\printAffiliationsAndNotice{}  

\begin{abstract}
Pre-training has been investigated to improve the efficiency and performance of training neural operators in data-scarce settings. However, it is largely in its infancy due to the inherent complexity and diversity, such as long trajectories, multiple scales and varying dimensions of partial differential equations (PDEs) data. In this paper, we present a new auto-regressive denoising pre-training strategy, which allows for more stable and efficient pre-training on PDE data and generalizes to various downstream tasks. Moreover, by designing a flexible and scalable model architecture based on Fourier attention, we can easily scale up the model for large-scale pre-training.We train our PDE foundation model with up to 1B parameters on 10+ PDE datasets with more than 100k trajectories. Extensive experiments show that we achieve SOTA on these benchmarks and validate the strong generalizability of our model to significantly enhance performance on diverse downstream PDE tasks like 3D data.
Code is available at \url{https://github.com/thu-ml/DPOT}.

\end{abstract}
\vspace{-1ex}
\section{Introduction}
\label{submission}

Learning solution operators for Partial Differential Equations (PDEs) is a fundamental task in scientific machine learning. It learns infinite-dimensional mappings from input to solution function spaces using parametrized PDE datasets to generalize to unseen inputs \cite{zachmanoglou1986introduction, karniadakis2021physics,li2020fourier}. Serving as surrogate models, neural operators can be more than five orders of magnitude faster~\cite{pathak2022fourcastnet,azzizadenesheli2023neural} than traditional numerical solvers and show great potential in applications such as weather forecasting\cite{pathak2022fourcastnet}, electromagnetism \cite{augenstein2023neural}, fluid dynamics \cite{li2022fourier}, and structural optimization \cite{shukla2024deep}. To improve the performance on different operator learning tasks, much effort has been devoted to designing effective model architectures like DeepONet~\cite{lu2021learning}, Fourier neural operator (FNO)~\cite{li2020fourier,li2022fourier} and transformers~\cite{cao2021choose,li2022transformer,hao2023gnot}. 

\begin{figure}[t]
    \centering
    \includegraphics[width=0.43\textwidth]{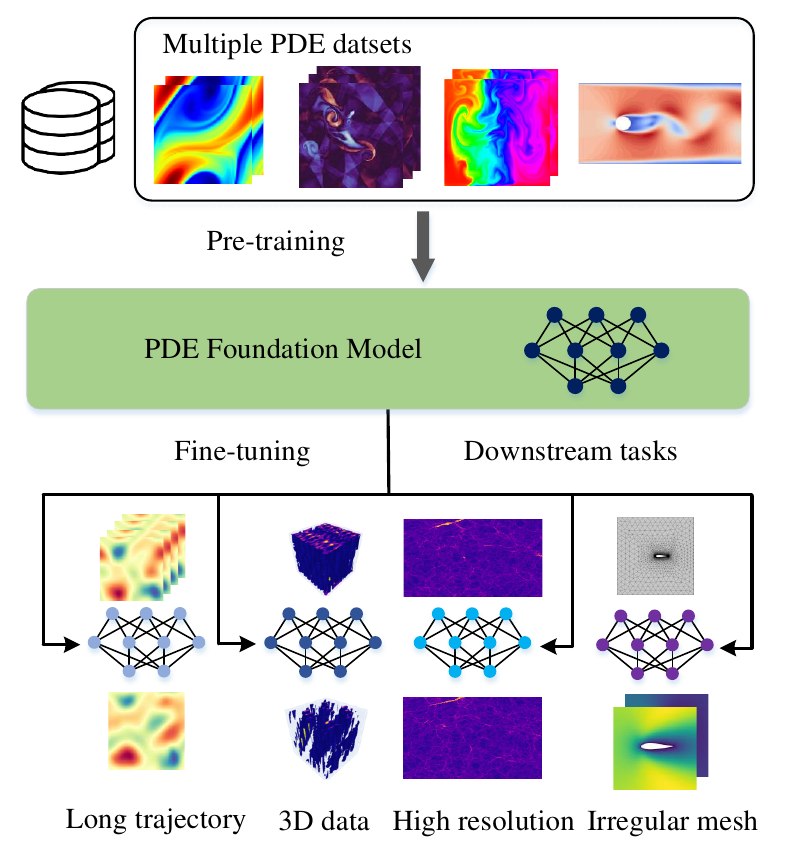}
    \vspace{-1em}
    \caption{An illustration of pre-training a PDE foundation model using massive data from multiple PDE datasets. The pre-trained model is then used for fine-tuning different downstream operator learning tasks, which can be complex. 
    (Best viewed in color)
    }
    \vspace{-1em}
    \label{fig0}
\end{figure}

A significant challenge for neural operators is their dependency on large datasets, typically obtained through costly simulations or experiments, which can limit their application in data-scarce problems. For instance, for an airfoil design problem, we usually need thousands of data samples to train a neural operator~\cite{shukla2024deep}, and each might take hours to days to generate. Large-scale pre-training has proven to be an effective paradigm with strong generalization ability to downstream tasks, in the data-rich fields like natural language processing (NLP) and computer vision (CV)~\cite{he2022masked,brown2020language}. 
However, large-scale pre-training of neural operators is still in its infancy. The gap from the fast development in NLP/CV is mainly due to the complexity and diversity of different PDE tasks, as highlighted in Fig.~\ref{fig0}. Firstly, different from images and text, PDE datasets exhibit significant variations in dimensions, number of temporal steps, resolutions, and geometric configurations. Secondly, different types and cases of PDEs have their own regularity and numerical ranges. This diversity necessitates the model and the pre-training strategy to be flexible to handle diverse inputs and scalable in its representational capacity to generalize to unseen data. 

Several works \cite{mialon2023self, mccabe2023multiple, subramanian2023towards} make initial attempts to pre-training PDE representations. However, these methods are not scalable and are limited to a narrow range of PDE tasks since their pre-training strategy and model architecture do not possess the required flexibility and capacity mentioned above. For instance, \citet{mialon2023self} uses contrastive learning, which relies on specially-designed data augmentation, while MPP \cite{mccabe2023multiple} pre-trains a vision transformer auto-regressively on PDEBench data~\cite{takamoto2022pdebench} and it might be unstable for long trajectories and cannot be transferred to data with diverse shapes.  \citet{subramanian2023towards} experimentally finds that data from different PDEs could be unified and mutually boost others for training models using operator frameworks but is limited to three simple PDE systems.

{\bf Our approach: }To resolve the challenges and push the limit of large-scale operator learning to diverse PDE tasks, we propose an auto-regressive Denoising Pre-training Operator Transformer (DPOT).
Firstly, we design an auto-regressive denoising pre-training strategy by injecting Gaussian noise into training data. By predicting the next timestep using noisy inputs, we show that the robustness and generalization ability greatly improve when transferring to downstream tasks. Secondly, to handle the complexity of PDE data, we introduce an efficient transformer backbone utilizing Fourier-based attention. Specifically, we apply a nonlinear learnable transform in frequency space, enablingefficient learning of kernel integral transforms for learning PDE solution maps. Finally, we collect massive PDE data, including more than 100k trajectories from more than ten datasets consisting of diverse PDEs, such as Navier-Stokes equations, diffusion-reaction equations, and shallow-water equations, with vastly different properties. We show that our DPOT can be easily scaled up from 7M to 0.5B (the largest currently available) as a foundational model for PDEs. We conduct extensive experiments on different downstream tasks to validate the generalization ability of the pre-trained model. Our contributions are summarized as:
\begin{itemize}[leftmargin=*]    
    \item We propose a new model architecture based on a Fourier transformer and pre-train it using a novel auto-regressive denoising strategy, enjoying excellent generalization and scaling ability. Based on this architecture and training scheme, we are able to pre-train the largest PDE foundational model with up to 0.5B parameters on more than 100k PDE data using 10+ datasets.
    \item  Based on our pre-trained model, we achieve state-of-the-art performance on multiple challenging benchmarks, such as PDEBench, PDEArena, and CFDBench, reducing error by up to 52\%. We also conduct massive experiments verifying the strong generalization and transfer ability on diverse downstream tasks out of training data, such as 3D PDE data, high-resolution data, and steady-state PDE.
\end{itemize}
\section{Related Work}
Here, we briefly summarize related work on neural operators
and pre-training in scientific machine learning.
\subsection{Neural Operators}
Neural operators learn solution operators of PDEs from extensive data, and they have shown immense potential in many fields such as fluid dynamics \cite{li2020fourier}, thermodynamics \cite{zhao2024recfno}, electromagnetism~\cite{augenstein2023neural} and climate forecasting \cite{pathak2022fourcastnet}. The diverse mathematical properties and formats of data across PDEs have spurred substantial research to design effective architectures of neural operators. For example, DeepONet \cite{lu2021learning} employs a branch network and a trunk network, while the Fourier Neural Operator (FNO) \cite{li2020fourier} introduces an efficient method by learning features in the frequency domain. Extensions of this method, such as Geo-FNO \cite{li2022fourier}, NUNO~\cite{liu2023nuno}, and GINO \cite{li2023geometry} adapt FNO to complex geometric shapes. Some work \cite{brandstetter2022message, michalowska2023neural} target time-dependent PDEs, focusing on next-time prediction challenges with specialized training strategies. PINO~\cite{li2021physics} and PI-DeepONet~\cite{wang2021learning} combine operator learning with PINNs~\cite{raissi2019physics,kharazmi2019variational,wang2021understanding,cuomo2022scientific} in training or fine-tuning that enhances the generalization ability.

Further progress has been made through new architectures utilizing transformers \cite{cao2021choose,li2022transformer,hao2023gnot,wu2023solving, ovadia2023vito}, integrating techniques like patchification and linear attention mechanism. For instance, the GK-Transformer \cite{cao2021choose}, OFormer \cite{li2022transformer}, and GNOT~\cite{hao2023gnot} have shown promise on problems with irregular geometries. AFNO~\cite{guibas2021adaptive} proposes efficient attention by applying FFT to the data and mixes features in the frequency domain inspired by FNO, which achieves low memory and computational cost similar to MLP-Mixer\cite{tolstikhin2021mlp}. It is then modified and deployed for large-scale climate forecasting \cite{pathak2022fourcastnet}.  Despite these significant efforts, they need to be individually trained for specific tasks and require much domain-specific data, highlighting the need to improve data efficiency in this field.

\subsection{Pre-training in Scientific Machine Learning}
Pre-training has proven to be a highly successful paradigm that enhances downstream tasks by training models in a (self-)supervised manner on large datasets. It not only achieves remarkable performance in traditional areas like natural language processing
 \cite{radford2018improving,radford2019language,brown2020language} and vision~\cite{he2020momentum,he2022masked}, but also shows potential in scientific machine learning like protein~\cite{jumper2021highly}, molecule~\cite{zhou2023uni}, climate and weather modeling~\cite{nguyen2023climax, mukkavilli2023ai,pathak2022fourcastnet}.

In the area of learning PDE data,  there have been initial attempts to explore pre-training across various physical systems. For example, \citet{mialon2023self} proposes contrastive learning on PDEs using symmetries. \citet{subramanian2023towards} design a comparatively universal PDE to train data from multiple steady-state PDEs collectively.~\citet{yang2023context} utilize the structure of MathGPT to explore in-context learning capabilities for PDE data.  Additionally, MPP~\cite{mccabe2023multiple} proposes an auto-regressive approach to pre-train on time-dependent PDE datasets. However, these explorations have so far been limited to specific types of equations and data. There remains much space to be explored in pre-training for more complex scenarios.

\section{Proposed Method}

\subsection{Overview of DPOT}
\begin{figure*}
    \centering
    \includegraphics[width=0.9\textwidth]{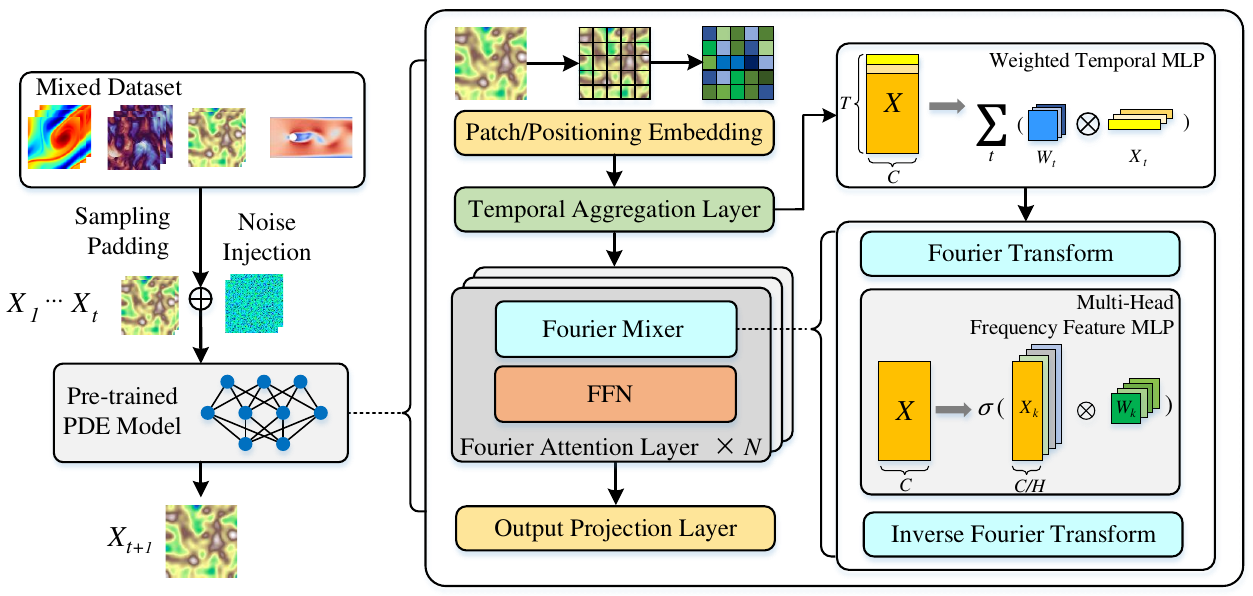}
    \vspace{-1em}
    \caption{An illustration of our model architecture. We first sample trajectories from mixed datasets of multiple PDEs. We optimize the model by predicting the next frame using noise-corrupted previous frames, which is also denoted as auto-regressive denoising training. We design a new model architecture consisting of a temporal aggregation layer and multiple Fourier attention layers. They can extract spatial-temporal features efficiently and can be easily scaled up to large models.
    (Best viewed in color)
    }
    \vspace{-1em}
    \label{fig1}
\end{figure*}

Here, we consider a general form of parametrized time-dependent PDEs with
variable $\tmmathbf{u} (x, t) \in \mathbb{R}^m$ defined on $(x, t) \in \Omega
\times T \subset \mathbb{R}^{d + 1}$ satisfying,
    \begin{eqnarray}
  \frac{\partial \tmmathbf{u}}{\partial t} - \mathcal{F}[\tmmathbf{u}; \theta] (x, t) &=& 0, (x,t) \in \Omega \times T  \\
  \tmmathbf{u} (x, 0)&=&\tmmathbf{u}^0 (x),\quad    x \in \Omega \nonumber\\
  \mathcal{B} [\tmmathbf{u}] (x, t)&=&0,\quad \quad \quad x \in \partial \Omega
  \nonumber
\end{eqnarray}
where $\mathcal{F}[\tmmathbf{u}, \theta](x, t) = F(t, x, \boldsymbol{u}, \partial_{x} \boldsymbol{u}, \partial_{x x}\boldsymbol{u},...;\theta)$ is a differential operator with spatial derivative terms, $\theta \in
\tmmathbf{\Theta}$ is an unknown parameter that determines the type and
coefficients of the PDE, $\tmmathbf{u}^0 (x)$ is the initial condition, and
$\mathcal{B} [\tmmathbf{u}] (x, t)$ denotes the boundary condition. Such initial-value problems are general to include multiple fundamental PDEs, for instance, Navier-Stokes
equations, diffusion-reaction equations, shallow-water equations, Burgers
equations, etc. 

In practice, we collect datasets containing data from different
PDEs $\mathcal{D}= \cup_{k = 1}^K \mathcal{D}_k$ where $\mathcal{D}_k = \{
\tmmathbf{u}_i \}_{1 \leqslant i \leqslant N_k}$. The solution functions
$\tmmathbf{u}_i \in \mathcal{D}$ are discretized on spatiotemporal meshes
$\tmmathbf{u}_i = (\tmmathbf{u}_i^1, \ldots, \tmmathbf{u}_i^T)$ and
$\tmmathbf{u}_i^t = \{ (x_j, u^t_j) : x_j \in \mathcal{X}_i \},  1 \leqslant t \leqslant T$. The meshes $\mathcal{X}_i$
could be a regular grid or irregular point clouds, depending on the geometry. The parameters $\theta$ decide the type and specific parameters of the PDE. However, in many practical applications like climate modeling and aerospace engineering, we only observe a trajectory of data but do not have access to the detailed parameters $\theta$ explicitly. In order to predict the next timesteps, we need to infer the most likely $\theta$ implicitly
from adjacent $T$ frames $(\tmmathbf{u}_i^1, \ldots, \tmmathbf{u}^T_i)$. We defer the details of datasets to Appendix \ref{dataset-details}.

We expect to extract suitable representations from trajectories of different PDEs $\mathcal{D}$ without the explicitly labeled $\theta$. There are several challenges to pre-training on these complex datasets. Firstly, these trajectories consist of data with uneven lengths on the temporal dimension. Traditional auto-regressive learning causes a distribution gap between training and testing~\cite{bengio2015scheduled, brandstetter2022message}. Here we introduce auto-regressive denoising pre-training by injecting noise into the data in Sec \ref{denoise}. We show that this regularizes the model and improves the robustness when testing. Secondly, these PDE datasets vary in resolution, shape, and the number of variables (channels). Such diversity and complexity necessitate the model to be flexible to input shapes and powerful in capacity. To achieve the goal, we present a preprocessing and sampling scheme in Sec \ref{preprocess}, and then we introduce a new transformer architecture in Sec \ref{model}. As we shall see, our model enjoys powerful expressivity and excellent scalability for learning from complex PDE data.

\subsection{Auto-regressive Denoising Pre-training}
\label{denoise}
To learn from temporal PDE datasets, we design a neural operator $\mathcal{G}_w
(\tmmathbf{u}^{t < T})$ parameterized by weights $w$ that auto-regressively
takes $T$ frames as input and decodes the next frame from previous frames, 
\begin{equation}
  \tmmathbf{u}^T =\mathcal{G}_w (\tmmathbf{u}^0, \ldots, \tmmathbf{u}^{T - 1})
  .
\end{equation}
By predicting the next frame, we learn an internal representation of PDEs.  However, directly supervising the one-step loss is not satisfactory due to the cumulative error when predicting a long trajectory at the testing stage, which is observed in both neural PDE solvers \cite{brandstetter2022message} and natural language processing \cite{bengio2015scheduled, zhang2019bridging}. Previous approaches like pushforward trick \cite{brandstetter2022message} increase the training complexity significantly, which is not suitable for pre-training.

Here we propose a simple yet effective strategy by injecting small-scale noise into the inputs. For $\forall t\leq T$, denote $\tmmathbf{u}^{<t}$ as $(\tmmathbf{u}^0, \ldots, \tmmathbf{u}^{t-1})$ and the noise as $\tmmathbf{\varepsilon} \sim \mathcal{N}(0, \epsilon||\tmmathbf{u}^{<t}||I)$. Then our goal is to predict the one-step transition between samples from different
datasets,
\begin{equation}
  \min_w \mathcal{L}= 
  \mathbb{E}_{\boldsymbol{u} \sim p(\mathcal{D})}  \sum_{1 \leqslant t\leqslant T} \| \mathcal{G}_w (\tmmathbf{u}^{<t} + \boldsymbol{\varepsilon}) -\tmmathbf{u}^{t} \|_2^2 ,
\end{equation}
where $p(\cdot)$ determines a sampling strategy on $\mathcal{D}$ which will be stated in Sec \ref{preprocess}.
By predicting the next timestep data from previous frames, the model learns to implicitly infer the details of the PDE, i.e., the parameters $\theta$, and propagate the solution to the next timestep. Experimentally we show that noise injection enhances the robustness and reduces the gap between training and inference.

\subsection{Data Preprocessing and Sampling}
\label{preprocess}
\textbf{Data Padding and Masking}.  Different PDE data vary in resolutions, numbers of variables, and geometric shapes. If we directly sample from the raw data, we obtain a batch with vastly different sizes, which leads to an unbalanced training load and decreases efficiency for modern multi-GPU training. Here, we propose a simple yet effective strategy to preprocess data in order to unify the shape of the data. Firstly, we choose a fixed resolution $H=128$, which matches a considerable part of the data. We upscale datasets with lower resolutions to $H$ through interpolation. For data with higher resolutions, we employ random sampling or interpolation to downscale to $H$. Secondly, to unify the number of variables (i.e., channels) across different PDEs, we pad all datasets (e.g., fill all entries with 1) along the channel dimension to match the dataset with the maximum number of channels. For datasets with irregular geometric shapes, we introduce an additional mask channel encoding the specific geometric configuration of each PDE instance.

\textbf{Balanced Data Sampling.} When training with multiple PDE datasets, the disparity between these datasets can lead to imbalanced training progress and low training efficiency. To address this issue, we propose to sample data by balancing the probability between datasets. We aim to balance the representation of each dataset within the training process. Let $|\mathcal{D}_k|$ be the number of data for the $k$-th dataset and $1\leqslant k \leqslant K$, we assign a weight $w_k$ for each dataset as an importance score for different datasets or PDEs. Then the probability of sampling from data from the $k$-th dataset is,
\begin{equation}
    p_k = \frac{w_k}{K |\mathcal{D}_k| \cdot \sum_k w_k}.
\end{equation}
We see that the probability of sampling from dataset $\mathcal{D}_k$ depends on the weight $w_k$ rather than its size $|\mathcal{D}_k|$. This alleviates the gradient imbalance caused by different dataset sizes. Detailed choice of $w_k$ is specified in Appendix \ref{train-details}.

\subsection{Model Architecture}
\label{model}

\textbf{Overview}.  Traditional transformers are inefficient for representing kernel transforms for diverse and high dimensional PDE data\cite{guibas2021adaptive}. Inspired by the ability of learning in spectral space of AFNO \cite{guibas2021adaptive}, we propose a new architecture based on Fourier attention as shown in Figure \ref{fig1}. First, we process the raw data through a patchifying layer and 

a temporal aggregation layer to reduce resolution and extract temporal dynamics inherent in PDEs. Then, we introduce a novel attention layer based on the Fourier mixer \cite{guibas2021adaptive}. It is capable of learning complex integral transformations within the frequency domain space, enabling the extraction of complex dynamic patterns. The proposed network enjoys strong expressivity and can easily be scaled up by increasing its depth and width for pre-training.

\textbf{Input Encoding and Temporal Aggregation}.
Suppose the input $\tmmathbf{u}^{< T} \in \mathbb{R}^{H \times W \times T
\times C}$ is a 2+1 dimensional spatiotemporal data with $C$ channels. We first encode the data with a patchification layer with positioning embedding similar to vision transformers\cite{dosovitskiy2020image},
\begin{equation}
  Z_p^t =\mathcal{P} (\tmmathbf{u}^t +\tmmathbf{p}^t), t = 1 \ldots T
\end{equation}
where $\mathcal{P}$ is a convolutional layer, $p_{i, j}^t = W_p (x_i, y_j, t)$ is a learnable positional encodings and $Z^t_p \in \mathbb{R}^{H\times W \times C}, W_p \in \mathbb{R}^{n \times 3}$. After that, we apply a novel temporal
aggregation layer to extract the information of the PDE from adjacent
timesteps. Specifically, for every node feature $\tmmathbf{z}^t_p \in \mathbb{R}^C$ in $Z^t_p = [\tmmathbf{z}^t_p]_{i, j}$,
we use a learnable transform $W_t$ with Fourier features $\tmmathbf{\gamma}
\in \mathbb{R}^C$,
\begin{equation}
  \tmmathbf{z}_{\tmop{agg}} = \sum_t W_t \cdummy \tmmathbf{z}_p^t e^{-
  i\tmmathbf{\gamma}t} .
\end{equation}
By aggregating the embeddings at different timesteps, we are able to extract
the information for inferring the type and parameters of PDE implicitly. However, due to
the highly complicated spatial dependency of PDE, we need to design an expressive layer to extract the spatial features.

\textbf{Fourier Attention Layer}. Here, we present our Fourier attention layer, which is the core layer for learning complex kernel integral transform to approximate the PDE solutions. Suppose $z^l
(x)$ is the features at $l$-th layer and $Z^l$ is the discretization of
feature $z^l (x)$. We apply the following kernel integral transform
$\mathcal{K}_{\phi}$ parameterized by a neural network to the features,
\begin{equation}
  (\mathcal{K}_{\phi} z^l) (x) = \int_{\Omega} \kappa (x, y ; \phi) z^l (y)
  \mathd y,
\end{equation}
where $\kappa (x, y ; \phi)$ is a neural network parameterized by parameters
$\phi$. However, the general form of kernel integral transform, like standard
attention takes quadratic complexity with respect to the spatial resolution
which is computationally expensive. To address this challenge, while
maintaining the strong expressivity of kernel integral transform, we adopt
to use a translation-invariant kernel $\kappa (x, y ; \phi) = \kappa (x - y ;
\phi)$. This reduced to a global convolution, which could be learned in Fourier
space,
\begin{equation}
  (\mathcal{K}_{\phi} z^l) (x) =\mathcal{F}^{- 1} [R_{\phi} \cdummy
  \mathcal{F} [z^l]] .
\end{equation}
where $z^l (x) \in \mathbb{R}^{d_z}$, and $R_{\phi} (k) \in \mathbb{C}^{d_z
\times d_z}$ is a function about frequency number. If we keep $m$ discretized
Fourier modes in practice, then we have $R_{\phi} \in \mathbb{C}^{m \times d_z
{\times d_z} } .$ This consumes a large amount of memory since we might keep
hundreds to thousands modes $m$ for high dimensional data. To improve the
utilization efficiency of parameters, we utilize the weight-sharing strategy
and approximate the frequency domain transformation using a shared MLP,
\begin{equation}
  \hat{z} (k) = W_2 \cdummy \sigma (W_1 \cdummy \mathcal{F} [z^l] (k) + b_1) +
  b_2 .
\end{equation}
Here $W_1, W_2 \in \mathbb{R}^{d_z \times d_z}, b_1, b_2 \in \mathbb{R}^{d_v}$
are learnable parameters, $\sigma (\cdummy)$ is an activation function. In
summary, we update the hidden representation $z^l (x)$ in a Fourier attention
layer using the following equation,
\begin{equation}
  z^{l + 1} (x) =\mathcal{F}^{- 1} [W_2 \cdummy \sigma (W_1 \cdummy
  \mathcal{F} [z^l] + b_1) + b_2] (x) .
\end{equation}
After that, we apply a group normalization layer\cite{wu2018group} for
better convergence speed. Then, we apply a feed-forward neural network
to the features to extract complex
dependencies at the channel dimension.

\textbf{Multi-head Structure.} To jointly attend to information from different representation subspaces.
We use a multi-head structure for the Fourier mixer. By splitting the whole features into multiple groups and learning them in different subspaces, we first divide spatial features $z^l (k) \in
\mathbb{R}^{d_z}$ into $h$ groups, i.e., $ z^l = \tmop{Concat} (z^l_1, z^l_2, \ldots z^l_h)$.

where $z^l_i (k) \in \mathbb{R}^{\frac{d_z}{h}}$ is the feature vector in each
block. Instead of applying a MLP to $z^l$, we use $h$ smaller MLPs to
transform each $z^l_i$ as follows,
\begin{equation}
  z^{l + 1}_i (x) =\mathcal{F}^{- 1} [W_{2, i} \cdummy \sigma (W_{1, i}
  \cdummy \mathcal{F} [z^l_i] + b_{1, i}) + b_{2, i}] (x),
\end{equation}
where $W_{1, i}, W_{2, i} \in \mathbb{R}^{d_z / h \times d_z / h}$ and $b_{1,
i}, b_{2, i} \in \mathbb{R}^{d_z / h}$. We see that the computation of
different MLPs could be fully parallelized which is efficient. The total
computation of the cost of multi-head Fourier attention will be $O (d_z^2 / h)$.

\subsection{Analysis of Our Model}

\textbf{Connections with AFNO.} Though our model draws inspiration from AFNO, there are many differences and improvements designed for multiple PDE pre-training. Firstly, we propose a time aggregation layer to extract PDE properties from adjacent frames. Secondly, we do not impose sparsity for token mixer (attention) layers, while AFNO uses soft-thresholding operations. Some PDEs could exhibit complex multi-scale features, and imposing sparsity in frequency domains might limit the expressivity of the model. A detailed comparison is in Appendix~\ref{compare-afno}.

\textbf{Theoretical Analysis.}
We present a theoretical analysis to show the expressivity of our model architecture. Specifically, we prove the following universal approximation theorem for Fourier attention layers, and the proof is deferred to Appendix~\ref{proof-thm}.
\begin{theorem}[Universal Approximation by Fourier attention layers]
    Let $s, s' > 0$; $\mathbb T^d = [0, 2\pi]^d$ be the d-dimensional torus;  $\mathcal G: H^s (\mathbb T^d; \mathbb R^{d_{in}} ) \to H^{s'} (\mathbb T^d; \mathbb R^{d_{out}})$ be a continuous operator between Sobolev spaces; and $K \subset H^s  (\mathbb T^d; \mathbb R^{d_{in}} )$ be a compact subset. Then, for any $\varepsilon >0$, there exists Fourier attention layers $\mathcal N: H^s (\mathbb T^d; \mathbb R^{d_{in}} ) \to H^{s'} (\mathbb T^d; \mathbb R^{d_{out}})$ satisfying:
\begin{equation}
    \sup _{v\in K} \|\mathcal G(v) -\mathcal N(v) \|_{L^2} \le \varepsilon .
\end{equation}
\end{theorem}
\vspace{-1ex}
\section{Experiments}
In this section, we first present the experimental setup. Then, we show the ability of our model to learn multiple PDEs and transfer them to diverse downstream tasks using extensive experiments.
Finally, we conduct ablation experiments on our model hyperparameters.
\subsection{Setup.}
\textbf{Datasets.} For the pre-training stage, we use 12 datasets from 4 different data sources, i.e., FNO~\cite{li2020fourier}, PDEBench~\cite{takamoto2022pdebench}, PDEArena~\cite{gupta2022towards}, and CFDBench~\cite{luo2023cfdbench}. The datasets employed in our study encompass diverse PDE types and parameters like Navier-Stokes equations, diffusion-reaction equations, and shallow-water equations. 

\textbf{Training and Evaluation.} 
 The model sizes are specified in Appendix \ref{train-details}. For models of all scales, we utilized the AdamW optimizer with a learning rate of $1 \times 10^{-3}$ and trained the models for 1000 epochs.
Our model was trained on servers equipped with eight A800 GPUs, each with 80 GB memory. Unless otherwise stated, we use $l_2$ relative error (L2RE) to measure the quality of the prediction, which is commonly used and can be referred to~\cite{li2020fourier}.

\begin{table*}[t]
\centering
\tiny
\begin{tabular}{cc|cccccccccccccc}
\hline
L2RE & Params & \multicolumn{3}{c|}{FNO-$\nu$} & \multicolumn{8}{c|}{PDEBench CNS-$(\eta, \zeta)$,DR,SWE} & \multicolumn{2}{c|}{PDEArena} & CFDBench \\
Subset & -- & 1e-5 & 1e-4 & \multicolumn{1}{c|}{1e-3} & 1,0.1 & 1,0.01 & M1 & 0.1,0.1 & 0.1,0.01 & M0.1 & DR & \multicolumn{1}{c|}{SWE} & NS & \multicolumn{1}{c|}{NS-cond} & -- \\ \hline
\multicolumn{2}{c|}{Small Model} &  &  &  &  &  &  &  &  &  &  &  &  &  &  \\
FNO & 0.5M & 0.156 & 0.0834 & 0.0128 & 0.098 & 0.096 & 0.097 & 0.360 & 0.170 & 0.265 & 0.12 & 0.0044 & 0.0912 & 0.319 & 0.00761 \\
UNet & 25M & 0.198 & 0.119 & 0.0245 & 0.334 & 0.291 & 0.313 & 0.569 & 0.357 & 0.463 & 0.0971 & 0.0521 & 0.102 & 0.337 & 0.209 \\
FFNO & 1.3M & 0.121 & 0.0503 & 0.0099 & 0.0212 & 0.052 & 0.0366 & 0.162 & 0.0452 & 0.104 & 0.0571 & 0.0116 & \textbf{0.0839} & 0.602 & 0.00714 \\
GK-T & 1.6M & 0.134 & 0.0792 & 0.0098 & 0.0341 & 0.0377 & 0.0359 & 0.0274 & 0.0366 & 0.0320 & 0.0359 & 0.00692 & 0.0952 & 0.423 & 0.0105 \\
GNOT & 1.8M & 0.157 & 0.0443 & 0.0125 & 0.0325 & 0.0420 & 0.0373 & 0.0228 & 0.0341 & 0.0285 & 0.0311 & 0.00678 & 0.172 & 0.325 & 0.00877 \\
Oformer & 1.9M & 0.1705 & 0.0645 & 0.0104 & 0.0417 & 0.0625 & 0.0521 & 0.0254 & 0.0205 & 0.0229 & 0.0192 & 0.00717 & 0.135 & 0.332 & 0.0102 \\
FNO-m & 7M & 0.116 & 0.0922 & 0.0156 & 0.151 & 0.108 & 0.130 & 0.230 & 0.076 & 0.153 & 0.0321 & 0.00912 & 0.210 & 0.384 & 0.0274 \\
MPP-Ti & 7M & -- & -- & -- & -- & -- & 0.0442 & -- & -- & 0.0312 & 0.0168 & 0.0066 & -- & -- & -- \\
MPP-S & 30M & -- & -- & -- & -- & -- & 0.0319 & -- & -- & 0.0213 & \textbf{0.0112} & \textbf{0.0024} & -- & -- & -- \\
\textcolor{dkred}{Ours-Ti} & 7M & 0.0976 & 0.0606 & \textbf{0.00954} & 0.0173 & 0.0397 & 0.0285 & 0.0132 & 0.0220 & 0.0176 & 0.0321 & 0.00560 & 0.125 & 0.384 & 0.00952 \\
\textcolor{dkred}{Ours-S} & 30M & \textbf{0.0553} & \textbf{0.0442} & 0.0131 & \textbf{0.0153} & \textbf{0.0337} & \textbf{0.0245} & \textbf{0.0119} & \textbf{0.0187} & \textbf{0.0153} & 0.0379 & 0.00657 & 0.0991 & \textbf{0.316} & \textbf{0.00696} \\ \hline
\multicolumn{2}{c|}{Pre-trained} &  &  &  &  &  &  &  &  &  &  &  &  &  &  \\
MPP-L & 400M & -- & -- & -- & -- & -- & 0.0208 & -- & -- & 0.0147 & \textbf{0.0098} & 0.00220 & -- & -- & -- \\
\textcolor{dkred}{Ours-S} & 30M & 0.0553 & 0.0442 & 0.0131 & 0.0153 & 0.0337 & 0.0245 & 0.0118 & 0.0188 & 0.0153 & 0.0379 & 0.00657 & 0.0999 & 0.316 & 0.00696 \\
\textcolor{dkred}{Ours-M} & 122M & 0.0409 & 0.0285 & 0.00474 & 0.0116 & 0.0238 & 0.0177 & \textbf{0.00866} & 0.0129 & 0.0108 & 0.0292 & 0.00290 & 0.0812 & 0.276 & 0.00752 \\
\textcolor{dkred}{Ours-L} & 500M & 0.0550 & 0.0274 & 0.00528 & 0.0100 & 0.0216 & 0.0158 & 0.00872 & 0.0115 & 0.0101 & 0.0232 & 0.00233 &0.0798 & 0.240 & \textbf{0.00650} \\ 
\textcolor{dkred}{Ours-H} & 1.03B & \cellcolor{tableblue}\textbf{0.0174} & \textbf{0.0131} & \cellcolor{tableblue}\textbf{0.00229} & \cellcolor{tableblue}\textbf{0.00961} & \textbf{0.0180} & \textbf{0.0138} & \cellcolor{tableblue}\textbf{0.00847} & \textbf{0.0105} & \cellcolor{tableblue}\textbf{0.00948} & 0.0191 & \textbf{0.00199} & \textbf{0.0379} & \textbf{0.213} & 0.00749 \\
\hline
\multicolumn{2}{c|}{DPOT-FT} &  &  &  &  &  &  &  &  &  &  &  &  &  &  \\
\textcolor{dkred}{T-200} & 7M & 0.0511 & 0.0431 & 0.00729 & 0.0136 & 0.0238 & 0.0187 & 0.0168 & 0.0145 & 0.0157 & 0.0194 & 0.00280 & 0.103 & 0.313 & 0.00537 \\
\textcolor{dkred}{S-200} & 30M & 0.0449 & 0.0342 & 0.00680 & 0.0152 & 0.0211 & 0.0182 & \textbf{0.0150} & 0.0151 & 0.0151 & 0.0171 & 0.00224 & 0.0892 & 0.290 & 0.00442 \\
\textcolor{dkred}{M-200} & 100M & 0.0255 & 0.0144 & 0.00427 & 0.0123 & 0.0179 & 0.0151 & 0.0182 & 0.0117 & 0.0149 & 0.0142 & 0.00218 & 0.0329 & 0.191 & 0.00452 \\
\textcolor{dkred}{L-200} & 500M & 0.0235 & 0.0117 & 0.00383 & 0.0114 & 0.0153 & 0.0133 & 0.0171 & 0.0108 & 0.0140 & 0.0158 & 0.00197 & 0.0307 & 0.182 & 0.00480 \\
\textcolor{dkred}{T-500} & 7M & 0.0520 & 0.0367 & 0.00580 & 0.0112 & 0.0195 & 0.0153 & 0.0174 & 0.0138 & 0.0156 & 0.0148 & 0.00241 & 0.0910 & 0.280 & 0.00391 \\
\textcolor{dkred}{S-500} & 30M & 0.0322 & 0.0237 & 0.00437 & 0.0129 & 0.0167 & 0.0148 & 0.0152 & 0.0126 & 0.0139 & 0.0129 & 0.00235 & 0.0867 & 0.268 & 0.00382 \\
\textcolor{dkred}{M-500} & 100M & 0.0229 & 0.0126 & 0.00335 & \textbf{0.00998} & 0.0146 & 0.0123 & 0.0161 & 0.00947 & 0.0128 & 0.0103 & 0.00227 & 0.0294 & 0.172 & 0.00373 \\
\textcolor{dkred}{L-500} & 500M & \textbf{0.0213} & \cellcolor{tableblue}\textbf{0.0104} & \textbf{0.00323} & 0.0108 & \cellcolor{tableblue}\textbf{0.0131} & \cellcolor{tableblue}\textbf{0.0119} & 0.0160 & \cellcolor{tableblue}\textbf{0.00905} & 
\textbf{0.0125} & 
\cellcolor{tableblue}\textbf{0.00739} & \cellcolor{tableblue}\textbf{0.00170} & \cellcolor{tableblue}\textbf{0.0278} & \cellcolor{tableblue}\textbf{0.170} & \cellcolor{tableblue}\textbf{0.00322} \\ \hline
\end{tabular}
\label{tb-main}
\caption{Results of main experiments are divided into three parts. Lower L2RE means better performance and we \textbf{bold} the best results in each part. We highlight the globally best results using \colorbox{tableblue}{blue} and the results of our DPOT using \textcolor{dkred}{darkred}. The first two parts are zero-shot results and the last part shows results for fine-tuning on each subset using our DPOT.}
\vspace{-1ex}
\end{table*}

\subsection{Baselines}
We selected the following highly influential methods as baselines for comparison. Except for the MPP-AViT/FNO-m are baselines that pre-trains and evaluated on all datasets according to \cite{mccabe2023multiple}, the errors of other methods are trained and evaluated on each sub-dataset individually.
\textbf{(Geo-)FNO}~\cite{li2020fourier,li2022fourier}: FNO represents an efficient framework for operator learning in the frequency domain. Geo-FNO extends its application to irregular grid datasets. For simplicity, we refer to both as FNO in the following discussions.
\textbf{UNet}~\cite{ronneberger2015u}: It is a classic network architecture designed for image-to-image mapping used in image segmentation and diffusion models.
\textbf{FFNO}~\cite{tran2021factorized}: It is an improvement over FNO by utilizing a separable Fourier representation. It reduces model complexity and facilitates deeper network structures.
\textbf{GK-Transformer/OFormer/GNOT}~\cite{cao2021choose,li2022transformer,hao2023gnot}: These are several transformer architectures based on linear attention for operator learning. They have different designs of their encoder/decoder and attention mechanisms.
\textbf{MPP-AViT/FNO-m}~\cite{mccabe2023multiple}: MPP-AViT involves auto-regressive pre-training on all datasets using a larger parameter AViT. For comparison, we train a larger parameter FNO (denoted as FNO-m), and evaluate its performance.

\subsection{Main Results}
Table \ref{tb-main} presents the results of our main experiments, i.e., the L2RE on the test sets of datasets in pre-training. The parameter in the second row means the parameter for the PDE dataset, e.g. $1e-5$ means the viscosity $\nu=1e-5$ in FNO dataset \cite{li2020fourier} and $1, 0.1$ means $(\eta, \zeta) = (1, 0.1)$ in PDEBench \cite{takamoto2022pdebench}.

The first part of the table shows the results of pre-trained and non-pre-trained small models ($\leqslant30$M). We see that our model successfully learns from multiple PDE datasets simultaneously using pre-training. We achieve state-of-the-art performance on 9 out of 12 datasets with an average improvement up to 52\% ($0.116$ to $0.0553$ on FNO-$1e-4$).
This indicates that our auto-regressive denoising pre-training strategy is effective for highly complex and heterogeneous PDE data. Specifically, our model outperforms most datasets with roughly equivalent model parameters. Compared to MPP, our model supports more datasets
and performs the best on most datasets except for two small subsets, i.e., PDB-SWE and PDB-DR. This demonstrates that given similar model sizes, the expressive power of our model's structure is more powerful in operator learning.

The second part of the table shows the zero-shot performance of larger pre-trained models ($\geqslant 30$M). It is observed that larger models significantly outperform smaller models. Specifically, our largest DPOT model shows notably better results compared to MPP-L and smaller DPOT models. 
This shows that our model scales better than the typical vision-based transformers. However, we also observe that on the PDB-DP subset, MPP-L achieves lower L2RE compared with our models, which shows there is room left for improvement.
Moreover, compared with DPOT-S, the DPOT-L model achieves better results on 10 datasets with improvement up to 39\% (on PDB-$0.1, 0.01$). The trend of consistent improvement when scaling up demonstrates potential for future application in pre-training on a larger scale.

The last part of the table displays the performance of differently sized pre-trained DPOTs after fine-tuning. The ``-200'' and ``-500'' suffixes denote fine-tuning on each subset for 200 and 500 epochs, respectively.  We see that larger models also yield better fine-tuning results. The pre-training plus fine-tuning of DPOT-L(0.5B) achieves SOTA in 9 out of 12 tasks. It reduces L2RE for more than 50\% on 8 out of 12 tasks compared with all zero-shot methods. This suggests that larger pre-trained models extract more knowledge from the datasets and provide significantly better initialization for fine-tuning. 
Moreover, we also observe that fine-tuning more steps is usually better than fewer steps, which shows there exists a trade-off between the computational cost and performance.

In summary, we demonstrate that our DPOT models show great advantages on these challenging benchmarks by using the auto-regressive denoising pre-training and scaling up model sizes. Moreover, results also show that the model architecture of DPOT is more scalable compared with vision-transformer architecture and traditional FNO models. By fine-tuning the DPOT-L model (0.5B), we achieve remarkably better results across these tasks. This indicates that systematically learning PDE representations from a large volume of PDE data is a promising and scalable direction, enhancing the performance and data efficiency of downstream operator learning tasks.

\subsection{Downstream Tasks Experiments}

\begin{table}[]
\scriptsize
\centering
\vspace{-1ex}
\begin{tabular}{c|ccc}
\hline
L2RE/L1-Medium$^*$ & Turbulence & 3D dataset & Steady$^*$ \\
Dataset & PDB-turb & PDB & CNO \\ \hline
(Geo-)FNO & 0.193 & 0.410 & 0.0357 \\
MPP-FT & 0.152 & -- & -- \\
DPOT-Vanilla & 0.167 & 0.262 & 0.0331 \\
DPOT-FT & \textbf{0.135} & \textbf{0.226} & \textbf{0.0230} \\ \hline
\end{tabular}
\caption{Experimental results of fine-tuning on downstream tasks. The first row shows the challenges of the downstream tasks. For steady-state PDE with $^*$, we adopt relative median $L^1$ error (L1-Median) following the original work~\cite{raonic2023convolutional}. Lower error means better performance, and we bold the best results.}
\label{tb-downstream}
\vspace{-3ex}
\end{table}

\begin{figure*}[t]
    \vspace{-1ex}
    \centering
    \begin{minipage}[t]{0.4\textwidth}
     \includegraphics[width=\textwidth]{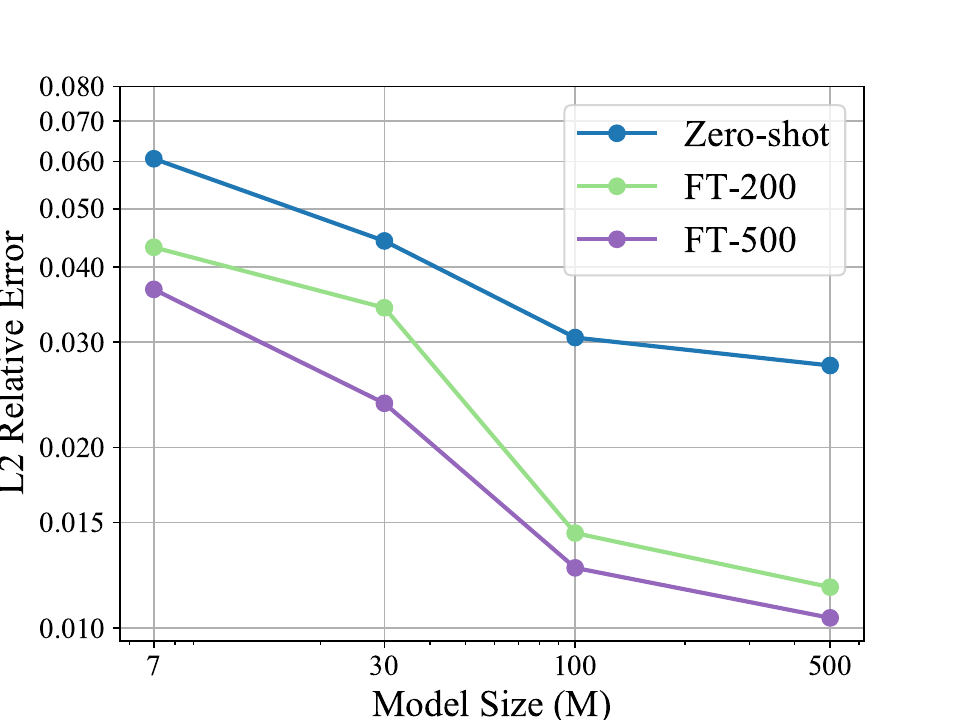}
    \end{minipage}
     \begin{minipage}[t]{0.4\textwidth}
     \includegraphics[width=\textwidth]{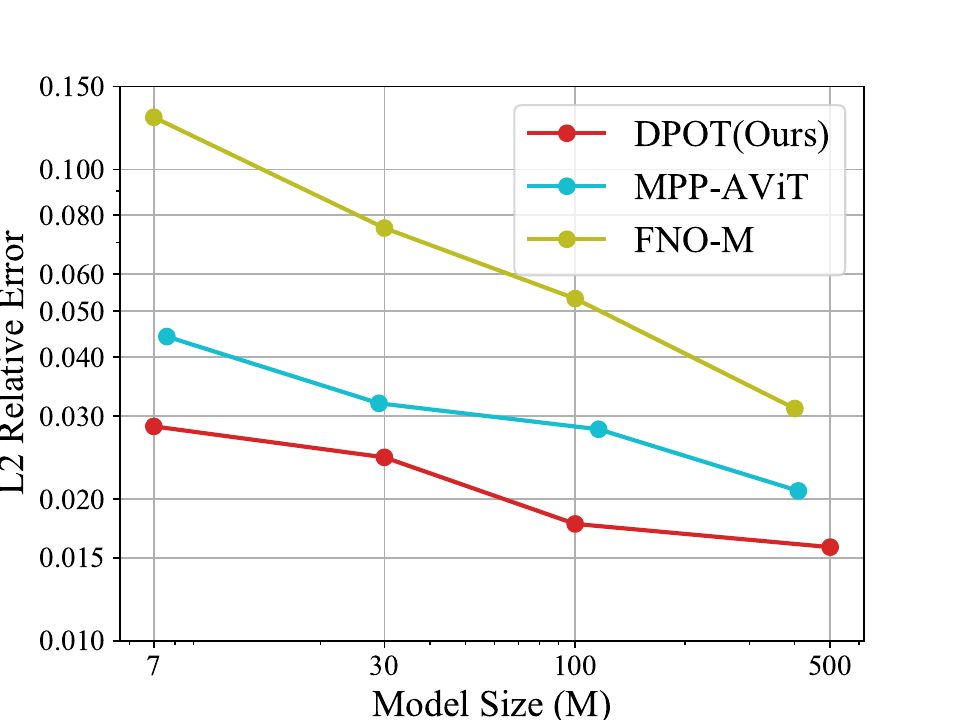}
    \end{minipage}
    \caption{Results of scaling experiments for different dataset sizes (left) and different numbers of layers (right).}
    \label{fig3}
\end{figure*}
To validate the effectiveness of our pre-trained model as a foundation for enhancing various types and complexities of PDE downstream tasks, we conducted the following experiments to explore the broader usage of the model. We transferred the weights of the Fourier attention block from our pre-trained model to the models for downstream tasks as initial values.
We choose our downstream tasks, including trajectory prediction for 2D high-resolution turbulence flow, 3D Navier-Stokes equations, and operator learning for steady-state PDEs. The results are displayed in Table \ref{tb-downstream}. The DPOT-Vanilla denotes the results of training the DPOT from scratch for 500 epochs, while DPOT-FT represents training for the same number of steps but using pre-trained weights for initialization. We have the following observations from the results.

 First, the results in the first column of the table represent the fine-tuning results on a high-resolution turbulence dataset with limited data\cite{takamoto2022pdebench}, compared with other methods. We observe that both MPP-FT and DPOT-FT achieve better results than training from scratch. Notably, our DPOT-FT performs better, which shows that it learns more effective representations from pre-training. Second, the results in the middle column of the table demonstrate DPOT's transferability to higher-dimensional datasets, which reduces error from $41\%$ to $22.6\%$. We find that even though only 2D data is used in pre-training, the learned representations still improve performance on 3D problems. Third, the last column shows that our model also achieves the best results when fine-tuned on single-step steady-state problems. These results validate the extraordinary versatility of our model, which can be easily extended to a wide range of downstream tasks.

\subsection{Scaling Experiments}

The change of performance with increasing sizes is a critical property of pre-trained models. 
Here, we conduct scaling experiments to verify the scalability.
We present the results in Figure \ref{fig3}. We observe that as the model size increases, the zero-shot test error consistently decreases, roughly following a scaling law. Specifically, from the left figure, when the model size is scaled from 7M to 500M parameters, the zero-shot L2RE reduces from $0.06$ to $0.028$. Additionally, fine-tuning the model on specific datasets results in enhanced performance. From the right figure, though all models enjoy scaling properties, our DPOT generalizes much better. Our 7M model achieves better performance compared with the 30M AViT model, showcasing better parameter efficiency.

\subsection{Ablation Experiments}
Here, we conduct ablation studies by training small-scale models on PDEBench datasets and compare the averaged zero-shot performance on corresponding test datasets. 

\begin{table}[]
\centering
\scriptsize
\begin{tabular}{c|cc|ccc}
\hline
$N_{h}$ & Train loss & Test L2RE & $P$ & Train loss & Test L2RE \\ \hline
1 & 0.0139  & 0.0214 & 2 & 0.0242 & 0.369 \\
4 & 0.0129  & 0.0186 & 4 & 0.0114 & 0.0275 \\
8 & 0.0127 & \textbf{0.0174} & 8 & 0.0127  & \textbf{0.0174} \\
16 & 0.0138 & 0.0188 & 16 & 0.0186 & 0.0278 \\ \hline
\end{tabular}
\label{tb-ab1}
\caption{Results of ablation experiments on the influences of the number of heads (left part) and patch sizes (right part).}
\vspace{-2ex}
\end{table}

\textbf{Influences of the number of heads/patch size.} 
The number of heads in our Fourier attention layer is a crucial hyperparameter.
This approach increases the diversity of the embeddings and reduces the computational complexity.
Moreover, the patch size $P$ in the encoder determines the spatial resolution size in the latent space. Here, we explore the effects of varying the number of heads and patch sizes by choosing from \{1, 2, 4, 8, 16\}. The results of these comparisons are presented in Table \ref{tb-ab1}. First, we find that choosing a medium number of heads, e.g., 4 or 8, leads to lower training loss, and 8 heads achieves the lowest testing loss. Second, on the right side, we see that patch size is a critical parameter that affects training loss and testing loss. We observe that choosing $P=8$ achieves significantly lower testing error. Additionally, choosing a small patch size like $P=2$ leads to poor performance.

\begin{table}[]
\vspace{-1ex}
\scriptsize
\centering
\begin{tabular}{c|cc}
\hline
Noise level $\varepsilon$ & Train loss & Test L2RE \\ \hline
0 & 0.00735 & 0.0178 \\
5E-5 & 0.00769 & \textbf{0.0152} \\
5E-4 & 0.0133 & 0.0156 \\
5E-3 & 0.0411 & 0.0672 \\
5E-2 & 0.166 & 0.753 \\ \hline
\end{tabular}
\label{tb-ab2}
\caption{Results of ablation experiments on the influences of the noise injection level $\varepsilon$.}
\vspace{-4ex}
\end{table}
\textbf{Influence of noise injection level}.
Here, we discuss the impact of the magnitude of noise injection on network training and generalization. As mentioned before, we observed that introducing noise into the input during auto-regressive training enhances the model's robustness and stability for long trajectories in the test set. We experiment with noise levels $\varepsilon$ from $5e-5$ to $5e-2$ and the results are shown in Table \ref{tb-ab2}. 
We find that the training error increases with the intensity of the injected noise, indicating that data with added noise becomes more challenging to learn.
However, the test error shows a tendency to decrease from $0.0178$ to $0.0152$ before increasing with the addition of noise.
We could view the noise injection as a regularization, and an appropriate level of noise benefits the generalization.

\vspace{-1ex}
\section{Conclusion}
In this paper, we proposed the DPOT for large-scale PDE pre-training. We designed the new auto-regressive denoising strategy and a new model architecture based on Fourier attention. It was scaled up to 1B on 10+ PDE datasets. Then, we conducted extensive experiments on many PDE benchmarks and downstream PDE tasks. Our work is a pioneering work in large-scale exploration of the scaling behavior of PDE pre-training. It shows great potential for future applications as downstream tasks in practical problems. 

\newpage
\section*{Broader Impact}
Neural operators are an important problem in scientific machine learning. Pre-training on PDE data offers a promising avenue to enhance the effectiveness and data efficiency of neural operators. This paper introduces DPOT, which leverages large-scale datasets from various PDEs for self-supervised learning, thereby broadly enhancing performance on downstream tasks. Such advancements hold the potential for application in real-world industrial manufacturing and scientific discovery. Regarding potential negative impacts, it is important to consider that neural network predictions of physical systems still entail errors and lack interpretability, which could pose risks in scenarios that demand high levels of safety and robustness. There are no serious ethical issues as this is basic research.

\bibliography{ref}

\begin{thebibliography}{51}
\providecommand{\natexlab}[1]{#1}
\providecommand{\url}[1]{\texttt{#1}}
\expandafter\ifx\csname urlstyle\endcsname\relax
  \providecommand{\doi}[1]{doi: #1}\else
  \providecommand{\doi}{doi: \begingroup \urlstyle{rm}\Url}\fi

\bibitem[Alberti et~al.(2023)Alberti, Dern, Thesing, and Kutyniok]{alberti2023sumformer}
Alberti, S., Dern, N., Thesing, L., and Kutyniok, G.
\newblock Sumformer: Universal approximation for efficient transformers.
\newblock In \emph{Topological, Algebraic and Geometric Learning Workshops 2023}, pp.\  72--86. PMLR, 2023.

\bibitem[Augenstein et~al.(2023)Augenstein, Repan, and Rockstuhl]{augenstein2023neural}
Augenstein, Y., Repan, T., and Rockstuhl, C.
\newblock Neural operator-based surrogate solver for free-form electromagnetic inverse design.
\newblock \emph{ACS Photonics}, 2023.

\bibitem[Azzizadenesheli et~al.(2023)Azzizadenesheli, Kovachki, Li, Liu-Schiaffini, Kossaifi, and Anandkumar]{azzizadenesheli2023neural}
Azzizadenesheli, K., Kovachki, N., Li, Z., Liu-Schiaffini, M., Kossaifi, J., and Anandkumar, A.
\newblock Neural operators for accelerating scientific simulations and design.
\newblock \emph{arXiv preprint arXiv:2309.15325}, 2023.

\bibitem[Bengio et~al.(2015)Bengio, Vinyals, Jaitly, and Shazeer]{bengio2015scheduled}
Bengio, S., Vinyals, O., Jaitly, N., and Shazeer, N.
\newblock Scheduled sampling for sequence prediction with recurrent neural networks.
\newblock \emph{Advances in neural information processing systems}, 28, 2015.

\bibitem[Brandstetter et~al.(2022)Brandstetter, Worrall, and Welling]{brandstetter2022message}
Brandstetter, J., Worrall, D., and Welling, M.
\newblock Message passing neural pde solvers.
\newblock \emph{arXiv preprint arXiv:2202.03376}, 2022.

\bibitem[Brown et~al.(2020)Brown, Mann, Ryder, Subbiah, Kaplan, Dhariwal, Neelakantan, Shyam, Sastry, Askell, et~al.]{brown2020language}
Brown, T., Mann, B., Ryder, N., Subbiah, M., Kaplan, J.~D., Dhariwal, P., Neelakantan, A., Shyam, P., Sastry, G., Askell, A., et~al.
\newblock Language models are few-shot learners.
\newblock \emph{Advances in neural information processing systems}, 33:\penalty0 1877--1901, 2020.

\bibitem[Cao(2021)]{cao2021choose}
Cao, S.
\newblock Choose a transformer: Fourier or galerkin.
\newblock \emph{Advances in neural information processing systems}, 34:\penalty0 24924--24940, 2021.

\bibitem[Cuomo et~al.(2022)Cuomo, Di~Cola, Giampaolo, Rozza, Raissi, and Piccialli]{cuomo2022scientific}
Cuomo, S., Di~Cola, V.~S., Giampaolo, F., Rozza, G., Raissi, M., and Piccialli, F.
\newblock Scientific machine learning through physics--informed neural networks: Where we are and what’s next.
\newblock \emph{Journal of Scientific Computing}, 92\penalty0 (3):\penalty0 88, 2022.

\bibitem[Dosovitskiy et~al.(2020)Dosovitskiy, Beyer, Kolesnikov, Weissenborn, Zhai, Unterthiner, Dehghani, Minderer, Heigold, Gelly, et~al.]{dosovitskiy2020image}
Dosovitskiy, A., Beyer, L., Kolesnikov, A., Weissenborn, D., Zhai, X., Unterthiner, T., Dehghani, M., Minderer, M., Heigold, G., Gelly, S., et~al.
\newblock An image is worth 16x16 words: Transformers for image recognition at scale.
\newblock \emph{arXiv preprint arXiv:2010.11929}, 2020.

\bibitem[Guibas et~al.(2021)Guibas, Mardani, Li, Tao, Anandkumar, and Catanzaro]{guibas2021adaptive}
Guibas, J., Mardani, M., Li, Z., Tao, A., Anandkumar, A., and Catanzaro, B.
\newblock Adaptive fourier neural operators: Efficient token mixers for transformers.
\newblock \emph{arXiv preprint arXiv:2111.13587}, 2021.

\bibitem[Gupta \& Brandstetter(2022)Gupta and Brandstetter]{gupta2022towards}
Gupta, J.~K. and Brandstetter, J.
\newblock Towards multi-spatiotemporal-scale generalized pde modeling.
\newblock \emph{arXiv preprint arXiv:2209.15616}, 2022.

\bibitem[Hao et~al.(2023)Hao, Wang, Su, Ying, Dong, Liu, Cheng, Song, and Zhu]{hao2023gnot}
Hao, Z., Wang, Z., Su, H., Ying, C., Dong, Y., Liu, S., Cheng, Z., Song, J., and Zhu, J.
\newblock Gnot: A general neural operator transformer for operator learning.
\newblock In \emph{International Conference on Machine Learning}, pp.\  12556--12569. PMLR, 2023.

\bibitem[He et~al.(2020)He, Fan, Wu, Xie, and Girshick]{he2020momentum}
He, K., Fan, H., Wu, Y., Xie, S., and Girshick, R.
\newblock Momentum contrast for unsupervised visual representation learning.
\newblock In \emph{Proceedings of the IEEE/CVF conference on computer vision and pattern recognition}, pp.\  9729--9738, 2020.

\bibitem[He et~al.(2022)He, Chen, Xie, Li, Doll{\'a}r, and Girshick]{he2022masked}
He, K., Chen, X., Xie, S., Li, Y., Doll{\'a}r, P., and Girshick, R.
\newblock Masked autoencoders are scalable vision learners.
\newblock In \emph{Proceedings of the IEEE/CVF conference on computer vision and pattern recognition}, pp.\  16000--16009, 2022.

\bibitem[Jumper et~al.(2021)Jumper, Evans, Pritzel, Green, Figurnov, Ronneberger, Tunyasuvunakool, Bates, {\v{Z}}{\'\i}dek, Potapenko, et~al.]{jumper2021highly}
Jumper, J., Evans, R., Pritzel, A., Green, T., Figurnov, M., Ronneberger, O., Tunyasuvunakool, K., Bates, R., {\v{Z}}{\'\i}dek, A., Potapenko, A., et~al.
\newblock Highly accurate protein structure prediction with alphafold.
\newblock \emph{Nature}, 596\penalty0 (7873):\penalty0 583--589, 2021.

\bibitem[Karniadakis et~al.(2021)Karniadakis, Kevrekidis, Lu, Perdikaris, Wang, and Yang]{karniadakis2021physics}
Karniadakis, G.~E., Kevrekidis, I.~G., Lu, L., Perdikaris, P., Wang, S., and Yang, L.
\newblock Physics-informed machine learning.
\newblock \emph{Nature Reviews Physics}, 3\penalty0 (6):\penalty0 422--440, 2021.

\bibitem[Kharazmi et~al.(2019)Kharazmi, Zhang, and Karniadakis]{kharazmi2019variational}
Kharazmi, E., Zhang, Z., and Karniadakis, G.~E.
\newblock Variational physics-informed neural networks for solving partial differential equations.
\newblock \emph{arXiv preprint arXiv:1912.00873}, 2019.

\bibitem[Li et~al.(2020)Li, Kovachki, Azizzadenesheli, Liu, Bhattacharya, Stuart, and Anandkumar]{li2020fourier}
Li, Z., Kovachki, N., Azizzadenesheli, K., Liu, B., Bhattacharya, K., Stuart, A., and Anandkumar, A.
\newblock Fourier neural operator for parametric partial differential equations.
\newblock \emph{arXiv preprint arXiv:2010.08895}, 2020.

\bibitem[Li et~al.(2021)Li, Zheng, Kovachki, Jin, Chen, Liu, Azizzadenesheli, and Anandkumar]{li2021physics}
Li, Z., Zheng, H., Kovachki, N., Jin, D., Chen, H., Liu, B., Azizzadenesheli, K., and Anandkumar, A.
\newblock Physics-informed neural operator for learning partial differential equations.
\newblock \emph{arXiv preprint arXiv:2111.03794}, 2021.

\bibitem[Li et~al.(2022{\natexlab{a}})Li, Meidani, and Farimani]{li2022transformer}
Li, Z., Meidani, K., and Farimani, A.~B.
\newblock Transformer for partial differential equations' operator learning.
\newblock \emph{arXiv preprint arXiv:2205.13671}, 2022{\natexlab{a}}.

\bibitem[Li et~al.(2022{\natexlab{b}})Li, Peng, Yuan, and Wang]{li2022fourier}
Li, Z., Peng, W., Yuan, Z., and Wang, J.
\newblock Fourier neural operator approach to large eddy simulation of three-dimensional turbulence.
\newblock \emph{Theoretical and Applied Mechanics Letters}, 12\penalty0 (6):\penalty0 100389, 2022{\natexlab{b}}.

\bibitem[Li et~al.(2023)Li, Kovachki, Choy, Li, Kossaifi, Otta, Nabian, Stadler, Hundt, Azizzadenesheli, et~al.]{li2023geometry}
Li, Z., Kovachki, N.~B., Choy, C., Li, B., Kossaifi, J., Otta, S.~P., Nabian, M.~A., Stadler, M., Hundt, C., Azizzadenesheli, K., et~al.
\newblock Geometry-informed neural operator for large-scale 3d pdes.
\newblock \emph{arXiv preprint arXiv:2309.00583}, 2023.

\bibitem[Liu et~al.(2023)Liu, Hao, Ying, Su, Cheng, and Zhu]{liu2023nuno}
Liu, S., Hao, Z., Ying, C., Su, H., Cheng, Z., and Zhu, J.
\newblock Nuno: A general framework for learning parametric pdes with non-uniform data.
\newblock \emph{arXiv preprint arXiv:2305.18694}, 2023.

\bibitem[Lu et~al.(2021)Lu, Jin, Pang, Zhang, and Karniadakis]{lu2021learning}
Lu, L., Jin, P., Pang, G., Zhang, Z., and Karniadakis, G.~E.
\newblock Learning nonlinear operators via deeponet based on the universal approximation theorem of operators.
\newblock \emph{Nature machine intelligence}, 3\penalty0 (3):\penalty0 218--229, 2021.

\bibitem[Luo et~al.(2023)Luo, Chen, and Zhang]{luo2023cfdbench}
Luo, Y., Chen, Y., and Zhang, Z.
\newblock Cfdbench: A comprehensive benchmark for machine learning methods in fluid dynamics.
\newblock \emph{arXiv preprint arXiv:2310.05963}, 2023.

\bibitem[McCabe et~al.(2023)McCabe, Blancard, Parker, Ohana, Cranmer, Bietti, Eickenberg, Golkar, Krawezik, Lanusse, et~al.]{mccabe2023multiple}
McCabe, M., Blancard, B. R.-S., Parker, L.~H., Ohana, R., Cranmer, M., Bietti, A., Eickenberg, M., Golkar, S., Krawezik, G., Lanusse, F., et~al.
\newblock Multiple physics pretraining for physical surrogate models.
\newblock \emph{arXiv preprint arXiv:2310.02994}, 2023.

\bibitem[Mialon et~al.(2023)Mialon, Garrido, Lawrence, Rehman, LeCun, and Kiani]{mialon2023self}
Mialon, G., Garrido, Q., Lawrence, H., Rehman, D., LeCun, Y., and Kiani, B.
\newblock Self-supervised learning with lie symmetries for partial differential equations.
\newblock In \emph{ICLR 2023 Workshop on Physics for Machine Learning}, 2023.

\bibitem[Micha{\l}owska et~al.(2023)Micha{\l}owska, Goswami, Karniadakis, and Riemer-S{\o}rensen]{michalowska2023neural}
Micha{\l}owska, K., Goswami, S., Karniadakis, G.~E., and Riemer-S{\o}rensen, S.
\newblock Neural operator learning for long-time integration in dynamical systems with recurrent neural networks.
\newblock \emph{arXiv preprint arXiv:2303.02243}, 2023.

\bibitem[Mukkavilli et~al.(2023)Mukkavilli, Civitarese, Schmude, Jakubik, Jones, Nguyen, Phillips, Roy, Singh, Watson, et~al.]{mukkavilli2023ai}
Mukkavilli, S.~K., Civitarese, D.~S., Schmude, J., Jakubik, J., Jones, A., Nguyen, N., Phillips, C., Roy, S., Singh, S., Watson, C., et~al.
\newblock Ai foundation models for weather and climate: Applications, design, and implementation.
\newblock \emph{arXiv preprint arXiv:2309.10808}, 2023.

\bibitem[Nguyen et~al.(2023)Nguyen, Brandstetter, Kapoor, Gupta, and Grover]{nguyen2023climax}
Nguyen, T., Brandstetter, J., Kapoor, A., Gupta, J.~K., and Grover, A.
\newblock Climax: A foundation model for weather and climate.
\newblock \emph{arXiv preprint arXiv:2301.10343}, 2023.

\bibitem[Ovadia et~al.(2023)Ovadia, Kahana, Stinis, Turkel, and Karniadakis]{ovadia2023vito}
Ovadia, O., Kahana, A., Stinis, P., Turkel, E., and Karniadakis, G.~E.
\newblock Vito: Vision transformer-operator.
\newblock \emph{arXiv preprint arXiv:2303.08891}, 2023.

\bibitem[Pathak et~al.(2022)Pathak, Subramanian, Harrington, Raja, Chattopadhyay, Mardani, Kurth, Hall, Li, Azizzadenesheli, et~al.]{pathak2022fourcastnet}
Pathak, J., Subramanian, S., Harrington, P., Raja, S., Chattopadhyay, A., Mardani, M., Kurth, T., Hall, D., Li, Z., Azizzadenesheli, K., et~al.
\newblock Fourcastnet: A global data-driven high-resolution weather model using adaptive fourier neural operators.
\newblock \emph{arXiv preprint arXiv:2202.11214}, 2022.

\bibitem[Radford et~al.(2018)Radford, Narasimhan, Salimans, Sutskever, et~al.]{radford2018improving}
Radford, A., Narasimhan, K., Salimans, T., Sutskever, I., et~al.
\newblock Improving language understanding by generative pre-training.
\newblock 2018.

\bibitem[Radford et~al.(2019)Radford, Wu, Child, Luan, Amodei, Sutskever, et~al.]{radford2019language}
Radford, A., Wu, J., Child, R., Luan, D., Amodei, D., Sutskever, I., et~al.
\newblock Language models are unsupervised multitask learners.
\newblock \emph{OpenAI blog}, 1\penalty0 (8):\penalty0 9, 2019.

\bibitem[Raissi et~al.(2019)Raissi, Perdikaris, and Karniadakis]{raissi2019physics}
Raissi, M., Perdikaris, P., and Karniadakis, G.~E.
\newblock Physics-informed neural networks: A deep learning framework for solving forward and inverse problems involving nonlinear partial differential equations.
\newblock \emph{Journal of Computational physics}, 378:\penalty0 686--707, 2019.

\bibitem[Raoni{\'c} et~al.(2023)Raoni{\'c}, Molinaro, Rohner, Mishra, and de~Bezenac]{raonic2023convolutional}
Raoni{\'c}, B., Molinaro, R., Rohner, T., Mishra, S., and de~Bezenac, E.
\newblock Convolutional neural operators.
\newblock \emph{arXiv preprint arXiv:2302.01178}, 2023.

\bibitem[Ronneberger et~al.(2015)Ronneberger, Fischer, and Brox]{ronneberger2015u}
Ronneberger, O., Fischer, P., and Brox, T.
\newblock U-net: Convolutional networks for biomedical image segmentation.
\newblock In \emph{Medical Image Computing and Computer-Assisted Intervention--MICCAI 2015: 18th International Conference, Munich, Germany, October 5-9, 2015, Proceedings, Part III 18}, pp.\  234--241. Springer, 2015.

\bibitem[Shukla et~al.(2024)Shukla, Oommen, Peyvan, Penwarden, Plewacki, Bravo, Ghoshal, Kirby, and Karniadakis]{shukla2024deep}
Shukla, K., Oommen, V., Peyvan, A., Penwarden, M., Plewacki, N., Bravo, L., Ghoshal, A., Kirby, R.~M., and Karniadakis, G.~E.
\newblock Deep neural operators as accurate surrogates for shape optimization.
\newblock \emph{Engineering Applications of Artificial Intelligence}, 129:\penalty0 107615, 2024.

\bibitem[Subramanian et~al.(2023)Subramanian, Harrington, Keutzer, Bhimji, Morozov, Mahoney, and Gholami]{subramanian2023towards}
Subramanian, S., Harrington, P., Keutzer, K., Bhimji, W., Morozov, D., Mahoney, M., and Gholami, A.
\newblock Towards foundation models for scientific machine learning: Characterizing scaling and transfer behavior.
\newblock \emph{arXiv preprint arXiv:2306.00258}, 2023.

\bibitem[Takamoto et~al.(2022)Takamoto, Praditia, Leiteritz, MacKinlay, Alesiani, Pfl{\"u}ger, and Niepert]{takamoto2022pdebench}
Takamoto, M., Praditia, T., Leiteritz, R., MacKinlay, D., Alesiani, F., Pfl{\"u}ger, D., and Niepert, M.
\newblock Pdebench: An extensive benchmark for scientific machine learning.
\newblock \emph{Advances in Neural Information Processing Systems}, 35:\penalty0 1596--1611, 2022.

\bibitem[Tolstikhin et~al.(2021)Tolstikhin, Houlsby, Kolesnikov, Beyer, Zhai, Unterthiner, Yung, Steiner, Keysers, Uszkoreit, et~al.]{tolstikhin2021mlp}
Tolstikhin, I.~O., Houlsby, N., Kolesnikov, A., Beyer, L., Zhai, X., Unterthiner, T., Yung, J., Steiner, A., Keysers, D., Uszkoreit, J., et~al.
\newblock Mlp-mixer: An all-mlp architecture for vision.
\newblock \emph{Advances in neural information processing systems}, 34:\penalty0 24261--24272, 2021.

\bibitem[Tran et~al.(2021)Tran, Mathews, Xie, and Ong]{tran2021factorized}
Tran, A., Mathews, A., Xie, L., and Ong, C.~S.
\newblock Factorized fourier neural operators.
\newblock \emph{arXiv preprint arXiv:2111.13802}, 2021.

\bibitem[Wang et~al.(2021{\natexlab{a}})Wang, Teng, and Perdikaris]{wang2021understanding}
Wang, S., Teng, Y., and Perdikaris, P.
\newblock Understanding and mitigating gradient flow pathologies in physics-informed neural networks.
\newblock \emph{SIAM Journal on Scientific Computing}, 43\penalty0 (5):\penalty0 A3055--A3081, 2021{\natexlab{a}}.

\bibitem[Wang et~al.(2021{\natexlab{b}})Wang, Wang, and Perdikaris]{wang2021learning}
Wang, S., Wang, H., and Perdikaris, P.
\newblock Learning the solution operator of parametric partial differential equations with physics-informed deeponets.
\newblock \emph{Science advances}, 7\penalty0 (40):\penalty0 eabi8605, 2021{\natexlab{b}}.

\bibitem[Wu et~al.(2023)Wu, Hu, Luo, Wang, and Long]{wu2023solving}
Wu, H., Hu, T., Luo, H., Wang, J., and Long, M.
\newblock Solving high-dimensional pdes with latent spectral models.
\newblock \emph{arXiv preprint arXiv:2301.12664}, 2023.

\bibitem[Wu \& He(2018)Wu and He]{wu2018group}
Wu, Y. and He, K.
\newblock Group normalization.
\newblock In \emph{Proceedings of the European conference on computer vision (ECCV)}, pp.\  3--19, 2018.

\bibitem[Yang et~al.(2023)Yang, Liu, Meng, and Osher]{yang2023context}
Yang, L., Liu, S., Meng, T., and Osher, S.~J.
\newblock In-context operator learning with data prompts for differential equation problems.
\newblock \emph{Proceedings of the National Academy of Sciences}, 120\penalty0 (39):\penalty0 e2310142120, 2023.

\bibitem[Zachmanoglou \& Thoe(1986)Zachmanoglou and Thoe]{zachmanoglou1986introduction}
Zachmanoglou, E.~C. and Thoe, D.~W.
\newblock \emph{Introduction to partial differential equations with applications}.
\newblock Courier Corporation, 1986.

\bibitem[Zhang et~al.(2019)Zhang, Feng, Meng, You, and Liu]{zhang2019bridging}
Zhang, W., Feng, Y., Meng, F., You, D., and Liu, Q.
\newblock Bridging the gap between training and inference for neural machine translation.
\newblock \emph{arXiv preprint arXiv:1906.02448}, 2019.

\bibitem[Zhao et~al.(2024)Zhao, Chen, Gong, Zhou, Yao, and Zhang]{zhao2024recfno}
Zhao, X., Chen, X., Gong, Z., Zhou, W., Yao, W., and Zhang, Y.
\newblock Recfno: a resolution-invariant flow and heat field reconstruction method from sparse observations via fourier neural operator.
\newblock \emph{International Journal of Thermal Sciences}, 195:\penalty0 108619, 2024.

\bibitem[Zhou et~al.(2023)Zhou, Gao, Ding, Zheng, Xu, Wei, Zhang, and Ke]{zhou2023uni}
Zhou, G., Gao, Z., Ding, Q., Zheng, H., Xu, H., Wei, Z., Zhang, L., and Ke, G.
\newblock Uni-mol: a universal 3d molecular representation learning framework.
\newblock 2023.

\end{thebibliography}
\bibliographystyle{icml2024}

\newpage
\appendix
\onecolumn

\section{Details of Pre-training and Fine-tuning}
\subsection{Model sizes and training details}
\label{train-details}
\textbf{Pre-training.}
We selected four models of varying sizes, i.e., DPOT-Tiny, DPOT-Small, DPOT-Medium, and DPOT-Large. The specific parameters of these models are shown in the table below. For the pre-training stage, we set the learning rate to $1 \times 10^{-3}$ and used a One-cycle learning rate schedule over 1000 epochs, with the first 200 epochs as the warm-up phase. The AdamW optimizer was employed with a weight decay of $1 \times 10^{-6}$ and momentum parameters $(\beta_1, \beta_2) = (0.9, 0.9)$. Training was conducted using eight A800 GPUs, with a total batch size of 160. The patch size was set to 8. For datasets with long trajectories, such as PDB-DR and PDB-SWE, we assigned a weight of $w=3$, while for others, the weight was set to 1.
For our training process, we selected $T = 10$ timesteps to predict the next frame, aligning with the original settings of most datasets.

\textbf{Fine-tuning}. Our model enjoys versatility for fine-tuning across a variety of downstream tasks involving different PDE datasets. The key module of our model is the Fourier attention layer, where the parameters are shared across different frequency components in the channel dimension. This shared parameterization enables efficient reuse of the model's parameters. Specifically, the FFT and IFFT operations within the model are replaced with their counterparts suitable for the specific data dimensions.  Additionally, other components, such as patch embedding and positioning embedding, are adjusted to align with the dimensionality of the target data. For instance, for 3D data, we could use existing parameters but replace the 2D FFT with 3D FFT.

\subsection{Details of Datasets}
\label{dataset-details}
Here, we list the PDEs of datasets we used for pre-training.
\begin{itemize}
    \item \textbf{FNO-$\nu$:} The quantity of interest(QoI) is the vorticity $w(x,t), (x, t) \in [0, 1]^2 \times [0, T]$ and it satisfies,
    \begin{eqnarray}
    \partial_t w + u \cdummy \nabla w & = & \nu \Delta w + f (x),\\
    \nabla \cdummy u & = & 0.
    \end{eqnarray}
    \item \textbf{PDEBench-CMS:} We need to predict the velocity, pressure, and density fields $\boldsymbol{u}(x,t), p(x,t), \rho(x,t)$ where $(x,t) \in [0, 1]^2 \times [0, 1]$. The PDEs are,
    \begin{eqnarray}
    \partial_t \rho + \nabla \cdummy (\rho \tmmathbf{u}) & = & 0,\\
    \rho (\partial_t \tmmathbf{u}+\tmmathbf{u} \cdummy \nabla \tmmathbf{u}) &
    = & - \nabla p + \eta \Delta \tmmathbf{u}+ (\varsigma + \eta / 3) \nabla
    (\nabla \cdummy \tmmathbf{u}),\\
    \partial_t \left( \frac{3}{2} p + \frac{\rho u^2}{2} \right) & = & -
    \nabla \cdummy \left( \left( \varepsilon + p + \frac{\rho u^2}{2} \right)
    \tmmathbf{u}-\tmmathbf{u} \cdummy \sigma' \right) .
  \end{eqnarray}
  \item \textbf{PDEBench-SWE:} We need to predict water depth $h(x,t)$ where the domain is $[- 1, 1]^2 \times [0, 5]$. The PDEs are as follows,
  \begin{eqnarray}
    \partial_t h + \nabla \cdummy (h\tmmathbf{u}) & = & 0,\\
    \partial_t (h\tmmathbf{u}) + \nabla \cdummy \left( \frac{1}{2}
    h\tmmathbf{u}^2 + \frac{1}{2} g_r h^2 \right) & = & - g_r h \nabla b.
  \end{eqnarray} 
  \item \textbf{PDEBench-DR:} We need to predict the density fields $\boldsymbol{u}(x,t)$. The domain is $[- 2.5, 2.5]^2 \times [0, 1]$ and the PDEs are,
  \begin{eqnarray*}
    \partial_t \tmmathbf{u} & = & \tmmathbf{D} \nabla^2
    \tmmathbf{u}+\tmmathbf{R} (\tmmathbf{u}).
  \end{eqnarray*}
  \item \textbf{PDEArena-NS1/2:} We need to predict the velocity, pressure, and density fields $\boldsymbol{u}(x,t), p(x,t), \rho(x,t)$ where $(x, t) \in [0, 32]^2 \times [0, 24]$. The PDEs are,
  \begin{eqnarray}
    \partial_t \tmmathbf{v} & = & -\tmmathbf{v} \cdummy \nabla \tmmathbf{v}+
    \mu \nabla^2 \tmmathbf{v}- \nabla p +\tmmathbf{f},\\
    \nabla \cdummy \tmmathbf{v} & = & 0.
  \end{eqnarray}
  \item \textbf{CFDBench:} We need to predict the velocity and pressure fields $\boldsymbol{u}(x,t), p(x,t)$. The domains are different as this is a dataset with irregular geometries. The PDEs are,
  \begin{eqnarray*}
    \partial_t (\rho \tmmathbf{u}) + \nabla \cdummy (\rho \tmmathbf{u}^2) & =
    & - \nabla p + \nabla \cdummy \mu (\nabla \tmmathbf{u}+ \nabla
    \tmmathbf{u}^T),\\
    \nabla \cdummy (\rho \tmmathbf{u}) & = & 0.
  \end{eqnarray*}

\end{itemize}

\begin{table}[h]
\centering
\begin{tabular}{c|ccccc}
\hline
Size & Attention dim & MLP dim & Layers & Heads & Model size \\ \hline
Tiny & 512 & 512 & 4 & 4 & 7M \\
Small & 1024 & 1024 & 6 & 8 & 30M \\
Medium & 1024 & 4096 & 12 & 8 & 122M \\
Large & 1536 & 6144 & 24 & 16 & 509M \\
Huge & 2048 & 8092 & 27 & 8 & 1.03B \\\hline
\end{tabular}
\label{model-size}
\caption{Configurations of DPOT of different sizes.}
\end{table}

\section{Supplementary Experimental Results}

\subsection{Results of DPOT with 1B Parameters}

We have introduced a model with \textbf{1 billion parameters (DPOT-Huge)}, making it the largest among all current PDE pre-training models and an order of magnitude larger than all other models. We tested its zero-shot performance, and the results, compared with MPP-Large and DPOT-L, are shown in the Table \ref{tb-1b}. We observed significant improvements with DPOT-H as the model's parameters increased, with performance boosts of over 50\% on some datasets (like FNO datasets, PDEArena-NS) compared to DPOT-L. Except for a slight decrease in performance on the CFDBench dataset, DPOT-H nearly dominates the results.

{
\begin{table*}
\centering
\begin{tabular}{c|ccc}
\hline
L2RE & MPP-L(0.4B) & DPOT-L (0.5B) & DPOT-H(1B) \\
\hline
FNO-1e-5         & --          & 0.0550        & \textbf{0.0174}  \\
FNO-1e-4         & --          & 0.0274        & \textbf{0.01314} \\
FNO-1e-3         & --          & 0.00528       & \textbf{0.00229} \\
PDB-(1, 0.1)     & --          & 0.0100        & \textbf{0.00961} \\
PDB-(1, 0.01)    & --          & 0.0216        & \textbf{0.0180}  \\
PDB-M1           & 0.0208      & 0.0158        & \textbf{0.0138}  \\
PDB-(0.1, 0.1)   & --          & 0.00872       & \textbf{0.00847} \\
PDB-(0.1, 0.01)  & --          & 0.0115        & \textbf{0.0105}  \\
PDB-M0.1         & 0.0147      & 0.0101        & \textbf{0.00948} \\
PDB-DR           & \textbf{0.0098}  & 0.0232        & 0.0191      \\
PDB-SWE          & 0.00220     & 0.00233       & \textbf{0.00199} \\
PDEArena-NS      & --          & 0.0798        & \textbf{0.0379}  \\
PDEArena-NS-cond & --          & 0.240         & \textbf{0.213}   \\
CFDBench         & --          & \textbf{0.00650}   & 0.00749     \\
\hline
\end{tabular}
\caption{Zero-shot performance comparison for 1B DPOT model.}
\label{tb-1b}
\end{table*}

}

\subsection{Timing experiments}
Here, we show the inference speed of different DPOTs. We measure the average time of predicting a single timestep data. The results are shown in Table
\ref{inf-time}.

\begin{table}[h]
\centering
\begin{tabular}{c|c}
\hline
Model & Inference time(ms) \\ \hline
DPOT-Ti & 0.408 \\
DPOT-S & 0.986 \\
DPOT-M & 1.88 \\
DPOT-L & 5.87 \\ \hline
\end{tabular}
\label{inf-time}
\caption{Reults of single step inference time for different models.}
\end{table}

\subsection{Comparison with AFNO architecture}
\label{compare-afno}
As mentioned before, we do not use sparsity constraints to regularize the frequency features. Here, we compare the performance of our DPOT and AFNO using the same settings of the ablation study in the main text. The result is shown in the following Table \ref{tb-afno}. We see that enforcing sparsity at the frequency domain makes the training loss significantly higher, which also worsens the test performance. This shows that the soft shrink operation might not be suitable for PDE's operator learning. 

\begin{table}[h]
\centering
\begin{tabular}{c|cc}
\hline
L2RE & Train loss & Test loss \\ \hline
AFNO & 0.0664 & 0.121 \\
DPOT & 0.0127 & 0.0174 \\ \hline
\end{tabular}
\caption{Results of the comparison between AFNO and DPOT's model structures.}
\label{tb-afno}
\end{table}

\subsection{Performance under varying resolutions}
An important property of a neural operator is its ability to generalize to different resolutions. Since there are convolutional layers (i.e, patchification layers) in DPOT, we need to slightly modify the convolutional layers following CNO \cite{raonic2023convolutional} to preserve its resolution generalization ability. We conduct experiments to measure the L2RE of DPOT under different resolutions. The datasets used are the same for ablation experiments. The results are shown in the Table \ref{tb-varyres}. 

\begin{table*}[]
\centering
\begin{tabular}{c|c}
\hline
Resolution & Test L2RE \\ \hline
32 & 0.0181 \\
48 & 0.0182 \\
72 & 0.0187 \\
96 & 0.0187 \\
128 & 0.0190 \\ \hline
\end{tabular}

\label{tb-varyres}
\caption{Results of test L2RE under different resolutions.}
\end{table*}

\subsection{Supplementary Experiments on Datasets with Long Trajectories}

We added a supplementary experiment on long trajectory problems using DPOT. We generated 1000 data with a trajectory length of 500 at a 128x128 resolution, using the \textbf{FNO-1e-3} dataset generation code. The model was trained and tested on subsets of 20, 50, 100, 200, and 500, recording the corresponding L2RE as shown in Table \ref{tb-long}. We observed that at very short step sizes, the improvement brought by pre-training is minimal due to already high prediction accuracy. However, as the trajectory length increases, the prediction difficulty significantly rises. The pre-trained DPOT's prediction error increases very slowly over time, whereas the error for the randomly initialized model grows quickly. This suggests that the pre-trained DPOT model can better utilize the prior knowledge of other PDEs to simulate the evolution of fluid over long trajectories.

\begin{table*}[]
\centering
\begin{tabular}{c|ccccc}
\hline
Num steps & 20 & 50 & 100 & 200 & 500 \\
\hline
DPOT & 0.00188 & 0.00192 & 0.00592 & 0.0241 & 0.0912 \\
DPOT-pretrained & \textbf{0.00148} & \textbf{0.00186} & \textbf{0.00335} & \textbf{0.0110} & \textbf{0.0385} \\
\hline
\end{tabular}
\caption{L2RE for DPOT and DPOT-pretrained at different trajectory steps}
\label{tb-long}
\end{table*}

\subsection{Additional Results on SuperBench Dataset}

We conducted additional experiments on the SuperBench dataset, specifically on the Navier-Stokes Kraichnan Turbulence Flow problem (Re=16000), with the task of upsampling data by 8x, i.e., from 128x128 to 1024x1024. Following the SuperBench paper, we used RFNE (lower score means better performance) as the metric for comparison. The results are as shown in the Table \ref{tb-superbench}, with SwinIR, WDSR, and EDSR results coming directly from the SuperBench original paper. We found that the non-pretrained DPOT network performs comparably to the current best model, EDSR, while the pre-trained DPOT model significantly outperforms EDSR. This demonstrates that pre-trained PDE models can enhance performance on downstream super-resolution tasks.

\begin{table*}[]
\centering
\begin{tabular}{c|c}
\hline
Model & RFNE (\%) \\
\hline
SwinIR & 0.80 \\
WDSR & 0.72 \\
EDSR & 0.57 \\
DPOT-vanilla & 0.59 \\
DPOT-pretrained & \textbf{0.46} \\
\hline
\end{tabular}
\caption{RFNE (\%) on the SuperBench dataset for the Navier-Stokes Kraichnan Turbulence (Re=16000) x8 Upsampling Task}
\label{tb-superbench}
\end{table*}

\subsection{Supplementary experiments on Kolmogorov turbulence flow}

Additionally, we also added a dataset for Kolmogorov turbulence flow, a dataset with notable chaotic behavior. We generated 128 training trajectories with a resolution of 256x256x200 for training and 16 for testing. The results are shown in Table \ref{tb-turb}. We found that the vanilla DPOT model and the FNO model have a full trajectory test error of 82.2\% and 104\%, indicating that the model almost learned nothing, especially as time progressed. In contrast, the DPOT-pretrained error was 33.5\%, significantly better than the non-pretrained DPOT model, suggesting it has learned some patterns of chaotic fluid evolution. It shows that pre-trained initialization is very helpful for complex, chaotic fluid simulations.

\begin{table*}[]
\centering
\begin{tabular}{@{}ccc@{}}
\toprule
L2RE                & Train loss (step-wise) & Test loss (total trajectory) \\ \midrule
FNO                 & 0.0931                 & 1.04                         \\
DPOT                & 0.0485                 & 0.822                        \\
\textbf{DPOT-pretrained} & \textbf{0.0296}        & \textbf{0.335}                \\ \bottomrule
\end{tabular}
\caption{Test results for Kolmogorov turbulence flow}
\label{tb-turb}
\end{table*}

\subsection{Comparison between Equal Data Sampling and Our Balanced Data Sampling}

We evaluated the impact of uniform sampling from all datasets combined versus our balanced sampling approach for pre-training. We retrained DPOT-Tiny ($\sim$7M) for comparison with the original model. The results are shown in Table \ref{tb-sampling}. Our sampling method resulted in more even convergence across different datasets. For instance, with equal sampling, some datasets showed very poor convergence, such as FNO-1e-5, PDB-DR, PDB-SWE, etc. Although the performance on a few datasets like PDB-M1 was slightly better with equal sampling, overall, balanced sampling outperformed equal sampling.

\begin{table*}[]
\centering
\begin{tabular}{c|cc}
\hline
DPOT-Tiny(L2RE) & Equal sampling & Balanced sampling \\
\hline
FNO-1e-5 & 0.177 & \textbf{0.0976} \\
FNO-1e-4 & 0.0701 & \textbf{0.0606} \\
FNO-1e-3 & 0.0327 & \textbf{0.00954} \\
PDB-(1, 0.1) & \textbf{0.0149} & 0.0173 \\
PDB-(1, 0.01) & \textbf{0.0376} & 0.0397 \\
PDB-M1 & \textbf{0.0265} & 0.0285 \\
PDB-(0.1, 0.1) & 0.0135 & \textbf{0.0132} \\
PDB-(0.1, 0.01) & 0.0236 & \textbf{0.0220} \\
PDB-M0.1 & 0.0185 & \textbf{0.0176} \\
PDB-DR & 0.170 & \textbf{0.0321} \\
PDB-SWE & 0.0225 & \textbf{0.00560} \\
PDEArena-NS & 0.126 & \textbf{0.125} \\
PDEArena-NS-cond & 0.407 & \textbf{0.384} \\
CFDBench & 0.00981 & \textbf{0.00952} \\
\hline
\end{tabular}
\caption{Performance comparison of DPOT-Tiny with Equal and Balanced Sampling.}
\label{tb-sampling}
\end{table*}

\subsection{Comparison between Self-attention and Fourier Mixer}

We compared transformers using standard self-attention as a mixer with our DPOT model using a Fourier mixer, with results shown in Table \ref{tb-mixer}. The datasets used are consistent with the paper's ablation study, combining subsets of the Compressible Navier-Stokes equations from PDEBench. We found that models with self-attention require more parameters and computational resources. However, they do not fit or generalize to PDE data as well as models based on the Fourier mixer. Compared to self-attention mixers, the Fourier mixer could better utilize the spectrum properties of PDE data and reduce test errors by about 25\%.

\begin{table*}
\centering
\begin{tabular}{c|cccc}
\hline
L2RE & \#N params & Flops per forward & Train loss & Test loss \\
\hline
Self attention & 52.1M & 88.45G & 0.0183 & 0.0238 \\
Fourier mixer & 30.8M & 75.44G & \textbf{0.0127} & \textbf{0.0174} \\
\hline
\end{tabular}
\caption{Performance comparison of Self-attention and Fourier mixer in DPOT model.}
\label{tb-mixer}
\end{table*}

\section{Proof of Universal Approximation of Fourier Attention}

In the following section, we will demonstrate the Universal Approximation capability of DPOT. This means that the model possesses ample expressive ability, capable of fitting a large class of operators and adapting to the required precision. Specifically, we show the core layers, i.e. Fourier attention layers have the capacity to approximate any continuous operators.

\subsection{Preliminaries and Definition}

Firstly, as DPOT mainly focuses on PDE problems defined on a rectangular domain, we will concentrate on functions defined on periodic torus $\mathbb T^d  = [0, 2\pi]^d$. 

Given the domain, we can define \textit{Fourier transform} for any function $v\in L^2(\mathbb T^d)$ as:

\begin{equation}
    \mathcal F(v)(\xi) = \frac 1{(2\pi)^d} \int _{\mathbb T^d} v(x) e^{-i\langle\xi, x\rangle} \text d x
\end{equation}

However, in the context of numerical calculation, we have to represent the function in a discrete way. In DPOT, we choose a regular mesh $\mathcal M := \{x_{\mathbf j}\} \subset \mathbb T^d$, where the index vector is denoted as $\mathbf j = (j_1,\cdots,j_d) \in \mathcal I$ with $\mathcal I = \{(j_1, \cdots,j_d) | \forall l \le d, j_l \le n\}$. Given the mesh, we can represent the function by its value on the mesh: $\{v_{\mathbf j} = v(x_{\mathbf j})\}$ and then define \textit {Discrete Fourier transform} as:

\begin{equation}
    \mathcal F(\{v_{\mathbf j}\})(\mathbf k) = \frac1{\sqrt{n^d}} \sum _{j_l \le n} v_{\mathbf j} e^{-\frac{2\pi i\langle \mathbf j, \mathbf k\rangle}{n}}, \forall \ \mathbf k \in \mathcal I
\end{equation}

Similarly, we can define \textit{Discrete Inverse Fourier Transform} as:

\begin{equation}
    \mathcal F(\{v_{\mathbf j}\})(\mathbf k) = \frac1{\sqrt{n^d}} \sum _{j_l \le n} v_{\mathbf j} e^{\frac{2\pi i\langle \mathbf j, \mathbf k\rangle}{n}}, \forall \ \mathbf k \in \mathcal I
\end{equation}

With the preparations in place, we can now formally define our model \textit{DPOT} mathematically.

\begin{definition}[DPOT]
A \textit{DPOT} is a neural operator $\mathcal N$ defined between input space $\mathcal X^n$ and output space $\mathcal Y^n$, with $\mathcal X \subset \mathbb R^{d_i}, \mathcal Y \subset \mathbb R^{d_o}$ being the range of input and output function. The structure of \textit{DPOT} with $L$ layers can be described as:

\begin{equation}
    \mathcal N = \mathcal Q \circ \mathcal L_L \circ \cdots \circ \mathcal L_1 \circ \mathcal R
\end{equation}

where $L_l$ is the $l$-th \textit{fourier attention layers} of DPOT, $\mathcal R$ is the \textit{lifting operator} and $\mathcal Q$ is the \textit{projection operator}, which will be defined below.

\end{definition}
\begin{definition}[fourier attention layers of DPOT]
    The $l$-th \textit{fourier attention layers} of DPOT $L_l: \mathbb R^{d_h} \to \mathbb R^{d_h}$ is a mapping between hidden space with the structure:
\begin{equation}
    \mathcal L_l(\{v_{\mathbf j}\})(\mathbf k) = v_{\mathbf k} + M_l\circ \mathcal F^{-1}\circ K_l \circ \mathcal F(\{v_{\mathbf j}\})
\end{equation}

Here, $M_l:\mathbb R^{d_h} \to \mathbb R^{d_h}$ is a mapping in hidden space, which is a 2-layer MLP with GELU activation $\sigma$. $K_l:\mathbb C^{d_h} \to \mathbb C^{d_H}$ is a mapping in the frequency domain of hidden space, which is also a 2-layer MLP with GELU activation $\sigma$.

\end{definition}

\begin{definition}[lifting operator of DPOT]

The \textit{lifting operator} of DPOT $\mathcal R$ is a mapping from the codomain of input function $\mathcal X$ into hidden space $\mathcal R^{d_h}$

\begin{equation}
    \mathcal R(\{v_{\mathbf j}\})(\mathbf k) = R(v_{\mathbf k}, \mathbf k), R \in C_\infty(\mathcal X \oplus \mathcal I, \mathbb R^{d_h}),  \forall \ \mathbf k\in \mathcal I
\end{equation}

where $R$ is the pointwise lifting function and $P$ is the positional encoding.

\end{definition}

\begin{definition}[projection operator of DPOT]
the \textit{projection operator} of DPOT $\mathcal Q$ is a  mapping from hidden space $\mathbb R^{d_h}$ into the codomain of output function $\mathcal Y$:
\begin{equation}
    \mathcal Q(\{v_{\mathbf j}\})(\mathbf k) = Q(v_{\mathbf k}), Q \in C_\infty(\mathbb R^{d_h}, \mathcal Y), \forall \ {\mathbf k} \in \mathcal I
\end{equation}

where $Q$ is the pointwise projection function.

\end{definition}

Then, we will introduce some key definitions and theorems used in the following proof. 

\begin{definition}[Equivariant Function] \cite{alberti2023sumformer}

For any sequence-to-sequence function $f:\mathcal X^n \to \mathcal Y^n$, with $\mathcal X, \mathcal Y \subset \mathbb R^d$, is called an equivariant function, only when it is \textit{equivariant} to the input order. That is, for any permutation $\pi: [n] \to [n]$, the function satisfies:

\begin{equation}
    f([x_{\pi(1)}, \cdots, x_{\pi(n)}]) = [f_{\pi(1)}(X), \cdots, f_{\pi(n)} (X)]
\end{equation}
\end{definition}

\begin{definition}[Sumformer] \cite{alberti2023sumformer}
Let $\mathcal X, \mathcal Y \subset \mathbb R^d$ and $d'\in \mathcal N$, and let two functions $\phi: \mathcal X \to \mathbb R^{d'}, \psi : \mathcal X \times \mathbb R^{d'} \to \mathcal Y$. A \textit{Sumformer} is a function $S:\mathcal X^n \to \mathcal Y^n$ which can be evaluated by:

\begin{equation}
    S([x_1, \cdots,x_n]) := [\psi(x_1, \Sigma), \cdots, \psi(x_n, \Sigma)]
\end{equation}

where,

\begin{equation}
    \Sigma := \sum _{k=1}^{n} \phi(x_k)
\end{equation}

\end{definition}

In \cite{alberti2023sumformer}, the authors demonstrates the following theorem:

\begin{theorem}[Universal Approximation by Sumformer]\label{thm:uni-appx-sum}
    For any \textit{Equivariant function} $f: \mathcal X^n \to \mathcal Y^n$ and any $\varepsilon > 0$, there exists a Sumformer $S$ that satisfies:
\begin{equation}
    \sup _{X\in \mathcal X^n} \| f(X) - S(X) \|_{\infty} < \varepsilon
\end{equation}
\end{theorem}

\subsection{Proof of the theorem}

In the following proof, we will demonstrate that there exists a \textit{DPOT} model $\mathcal N$ that can approximate any operator $\mathcal G$ on a $d$-dimensional torus $\mathbb T^d$ with mesh $\mathcal M$. That is, for any $v\in L_2(\mathbb T^d)$ and $\varepsilon > 0$, we have:

\begin{equation}
    \|\mathcal N(\{v_{\mathbf j}\}) - \mathcal G(\{v_{\mathbf j}\}) \|_{\infty} < \varepsilon
\end{equation}

where $\{v_{\mathbf j}\} = \{v(x_{\mathbf j})\}_{j\in \mathcal I}$ is the value of $v$ on the mesh $\mathcal M \subset \mathbb T^d$.

The proof will consist of 3 steps. We will first construct a proxy target that is equivariant to input order, then construct a Sumformer to approximate the proxy leveraging Theorem \ref{thm:uni-appx-sum}. Finally, we will show that there exists a corresponding DPOT that can fit the constructed Sumformer.

\textbf{Step 1.} Construction of proxy target.

As the first preparatory step, we will first assign a integer identifier to each point in the grid, thereby treating it as a 1 dimensional sequence. To be specific, let $N = n^d$ and we can map any indices $\mathbf j\in \mathcal I$ to an integer by:
\begin{align}
\begin{aligned}
    &q: \mathcal I \to \mathbb [0, N]\\
    &q(\mathbf j) := \sum _{l=1}^n (j_l - 1)\cdot n^{l-1}
\end{aligned}
\end{align}

In the subsequent proof, we will use new integer indices $j\in [0, N]$ to replace the original vector indices $\mathbf j \in \mathcal I$.

Next, in order to leverage Theorem \ref{thm:uni-appx-sum}, we will construct an equivariant function by concatenating a permutation into the input. To be specific, let $\mathcal {\tilde X} = X \oplus \mathbb N$ and $U = \{[(x_1, \pi_1), \cdots, (x_N, \pi_N)] | x_j\in \mathcal X, \pi \in [N]\} \subset \mathcal {\tilde X}$, where $[N]$ is a permutation of $\{1,2,...,N\}$ the proxy target can be defined as:

\begin{align}
\begin{aligned}
    &f: U \to \mathcal Y^N\\
    &f( [(x_1, \pi(1)), \cdots, (x_N, \pi(N))] ) := \pi^{-1} \circ \mathcal G([x_{\pi^{-1}(1)}, x_{\pi^{-1}(2)}, \cdots, x_{\pi^{-1}(N)}]) 
\end{aligned}
\end{align}

where $\pi^{-1}$ is the inverse of $\pi$ in the permutation space. 

Now, we will show the equivariance of $f$ by the following lemma:

\begin{lemma}[equivariance of $f$]
    For any permutation $\sigma \in [N]$ and any input sequence $\{(x_j, \pi_j)\}\in U$, the proxy target $f$ satisfies:
    \begin{equation}
        f([(x_{\sigma(1)}, \pi_{\sigma(1)}), \cdots, (x_{\sigma(n)}, \pi_{\sigma(N)})]) = \sigma^{-1} \circ f( [(x_1, \pi_1), \cdots, (x_N, \pi_N)] )
    \end{equation}
\end{lemma}
\begin{proof}
    Let $\pi' = \pi \circ \sigma$, it can be shown that:
\begin{align}
\begin{aligned}
    &f([(x_{\sigma(1)}, \pi_{\sigma(1)}), \cdots, (x_{\sigma(N)}, \pi_{\sigma(N)})]) \\
    =&f([(x_{\sigma(1)}, \pi'(1)), \cdots, (x_{\sigma(N)}, \pi'(N))]) \\
    =&\pi'^{-1} \circ \mathcal G([x_{\sigma(\pi'^{-1}(1))}, \cdots,x_{\sigma(\pi'^{-1}(N))}])\\
    =&\sigma^{-1} \circ \pi^{-1} \circ \mathcal G([x_{\pi^{-1}(1)}, x_{\pi^{-1}(2)}, \cdots, x_{\pi^{-1}(N)}])\\
    =&\sigma^{-1} \circ f([(x_1, \pi_1), \cdots, (x_N, \pi_N)])
\end{aligned}
\end{align}
\end{proof}

\textbf{Step 2.} construct a Sumformer.

Given the equivariant sequence-to-sequence function $f$, we can now leverage Theorem \ref{thm:uni-appx-sum} to construct a Sumformer $\mathcal S$ satisfies:

\begin{equation}
    \| f( [(x_1, \pi(1)), \cdots, (x_N, \pi(N))]) - \mathcal S([(x_1, \pi(1)), \cdots, (x_N, \pi(N))])\|_{\infty} < \varepsilon
\end{equation}

for every $\{(x_j, \pi(j))\}_{j\le N} \in U$ and $\varepsilon > 0$

\textbf{Step 3.} using DPOT to approximate the Sumformer.

In the following step, we will construct a DPOT to approximate the Sumformer $\mathcal S$ in a subset of $U$.

\begin{lemma}
    Let $V = \{[(x_1, 1),\cdots,(x_N,N)]| x_i \in \mathcal X\} \cong \mathcal X^N$ be a subset of $U$. Let $p: \mathcal X^N \to V$ be the projection mapping between two spaces. Then there exists a DPOT $\mathcal N$ that:
    \begin{equation}
        \|\mathcal N(\{x_i\}) - \mathcal S(p(\{x_i\})) \|_{\infty} < \varepsilon
    \end{equation}
    for every $\{x_i\}\in \mathcal X^N$ and $\varepsilon > 0$
\end{lemma}

\begin{proof}
    In the upcoming proof, we will delve into the structure of DPOT and demonstrate how it can be employed to approximate the $\phi$ and $\psi$ within Sumformer.

    Firstly, we will construct the lifting operator $\mathcal R$ as:
\begin{align}
\begin{aligned}
    R(x_j, j) &= \begin{bmatrix}\phi(x_k)&\mathbf 0_{h_i}&P(j)\end{bmatrix} \in \mathbb R^{2h_i+1}\\
    \mathcal R(\{x_j\}) &= \begin{bmatrix} \phi(x_1)& \mathbf 0_{h_i}& P(1) \\ \cdots \\ \phi(x_N)&\mathbf 0_{h_i}& P(N)\end{bmatrix} \in \mathbb R^{N \times ( 2d_i + 1)}
\end{aligned}
\end{align}

with the positional encoding $P(j) = \mathcal F^{-1} (\{1,\cdots, N\}) (j)$. It is worth mentioning that, although we use the integer indices to replace to original vector indices, the discrete fourier transform is still defined on $\mathbb T^d$, not on a 1-dimensional sequence.

Then, in the first Fourier attention layer $\mathcal L_1$, we will first apply \textit{discrete Fourier Transform} on the sequence, yielding:

\begin{equation}
    \begin{bmatrix}
        \frac1{\sqrt N}\sum_{k=1}^N \phi(x_k) & \mathbf 0_{d_i}& 1 \\
        \cdots\\
        \frac1{\sqrt N}\sum_{k=1}^N \phi(x_k) e^{\frac{2\pi i \langle \mathbf 1, \mathbf k\rangle}{n}} & \mathbf 0_{d_i}& N
    \end{bmatrix} \in \mathbb R^{N \times ( 2d_i + 1)}
\end{equation}

After the Fourier transform, we will apply a pointwise 2-Layer MLP $\mathcal K_1$. According to the Universal Approximation Theorem of MLPs, we can find a $ K_1$ that sufficiently fits the following function:

\begin{equation}
    K_1(\begin{bmatrix}\mathbf z& k\end{bmatrix}) = 
    \begin{cases}
    \begin{bmatrix}\mathbf z& k\end{bmatrix} & \text{if } k = 1\\
    \begin{bmatrix}0& k\end{bmatrix} & \text{if } k \ne 1\\
    \end{cases}
\end{equation}

for any $\mathbf z\in \mathbb R^{2d_i}$.

Thus, the sequence will be mapped to 

\begin{equation}
K_1(\mathcal{R}(\{x_k\})) = 
    \begin{bmatrix}
        \frac1{\sqrt N}\sum_{k=1}^N \phi(x_k) & \mathbf 0_{d_i}& 1 \\
        \cdots\\
        \mathbf 0_{d_i} & \mathbf 0_{d_i}& N
    \end{bmatrix} \in \mathbb R^{N \times ( 2d_i + 1)}
\end{equation}

by $ K_1$, and be transformed to

\begin{equation}
\mathcal{F}^{-1}(K_1(\mathcal{R}(\{x_k\}))) = 
    \begin{bmatrix}
        \frac1N\Sigma & \mathbf 0_{d_i}& P(1) \\
        \cdots\\
        \frac1N\Sigma & \mathbf 0_{d_i}& P(N)
    \end{bmatrix} \in \mathbb R^{N \times ( 2d_i + 1)}
\end{equation}

after \textit{discrete inverse Fourier Transform}. Here, $\Sigma: = \sum_{k=1}^N \phi(x_k)$.

The following mapping $M_1$ can be constructed as:

\begin{equation}
    M_1(\begin{bmatrix} \frac1N\Sigma &\mathbf 0_{d_i}& P(k)\end{bmatrix}) = \begin{bmatrix} \mathbf 0_{d_i}&\Sigma& P(k)\end{bmatrix}
\end{equation}

so after the skipping connection, the sequence will finally become:

\begin{equation}
    \begin{bmatrix}
        x_1 & \Sigma& 2P(1) \\
        \cdots\\
        x_N & \Sigma& 2P(N)
    \end{bmatrix} \in \mathbb R^{N \times ( 2d_i + 1)}
\end{equation}

In the next Fourier layer $\mathcal L_2$, we kept the kernel $K_2$ as the identity mapping, ensuring the (inverse) Fourier transform has no effect on the sequence. Then we utilize the following MLP $M_2$ to approximate $\psi$:

\begin{equation}
    M_2(\begin{bmatrix} \Sigma & x_k& 2P(k)\end{bmatrix}) = \psi(\Sigma, x_k)
\end{equation}

which is the desired output of Sumformer.

\end{proof}

\section{Limitations of the work.}
This work, as an empirical and numerical study, is capable of solving challenging PDE problems with high precision. However, due to the intricate mathematical nature of neural networks and the relative scarcity of theoretical research, employing them in risk-sensitive scenarios may pose safety risks.

\end{document}